\documentclass[10pt,twocolumn,letterpaper]{article}

\usepackage{cvpr}      

\usepackage[table]{xcolor}     
\usepackage{booktabs}          
\usepackage{array}              
\usepackage{graphicx}    
\usepackage{multirow} 

\usepackage{dashrule}
\usepackage{xcolor}

\usepackage{amsmath}  
\usepackage{amssymb}  
\usepackage{amsthm}   

\usepackage{marvosym} 
\usepackage{footmisc} 

\definecolor{cvprblue}{rgb}{0.21,0.49,0.74}
\usepackage[pagebackref,breaklinks,colorlinks,allcolors=cvprblue]{hyperref}

\def\paperID{5725} 
\def\confName{CVPR}
\def\confYear{2026}
\title{DiP: Taming Diffusion Models in Pixel Space}

\author{Zhennan Chen$^{1,2}$\thanks{Work done during the internship at Tencent. $^\dagger$Project Leader. $^\ddagger$Corresponding Author.}  ~ Junwei Zhu$^2$$^\dagger$ ~ Xu Chen$^{2}$ ~ Jiangning Zhang$^2$ ~ Xiaobin Hu$^3$ \\
	Hanzhen Zhao$^3$ ~ Chengjie Wang$^2$ ~ Jian Yang$^1$ ~ Ying Tai$^1$$^\ddagger$ \\
	$^1$Nanjing University  ~~  $^2$Tencent Youtu Lab~~ $^3$National University of Singapore ~~ \\
    {\small \textcolor{magenta}{\url{https://github.com/NJU-PCALab/DiP}}}
    \vspace{-8mm}
}

\begin{document}
\maketitle

\begin{abstract}

Diffusion models face a fundamental trade-off between generation quality and computational efficiency. Latent Diffusion Models (LDMs) offer an efficient solution but suffer from potential information loss and non-end-to-end training. In contrast, existing pixel space models bypass VAEs but are computationally prohibitive for high-resolution synthesis. To resolve this dilemma, we propose DiP, an efficient pixel space diffusion framework. DiP decouples generation into a global and a local stage: a Diffusion Transformer (DiT) backbone operates on large patches for efficient global structure construction, while a co-trained lightweight Patch Detailer Head leverages contextual features to restore fine-grained local details. This synergistic design achieves computational efficiency comparable to LDMs without relying on a VAE. DiP is accomplished with up to 10$\times$ faster inference speeds than previous method while increasing the total number of parameters by only 0.3\%, and achieves an 1.79 FID score on ImageNet 256$\times$256. 

\end{abstract}    
\begin{figure}[h]
    \centering
    \includegraphics[width=\linewidth]{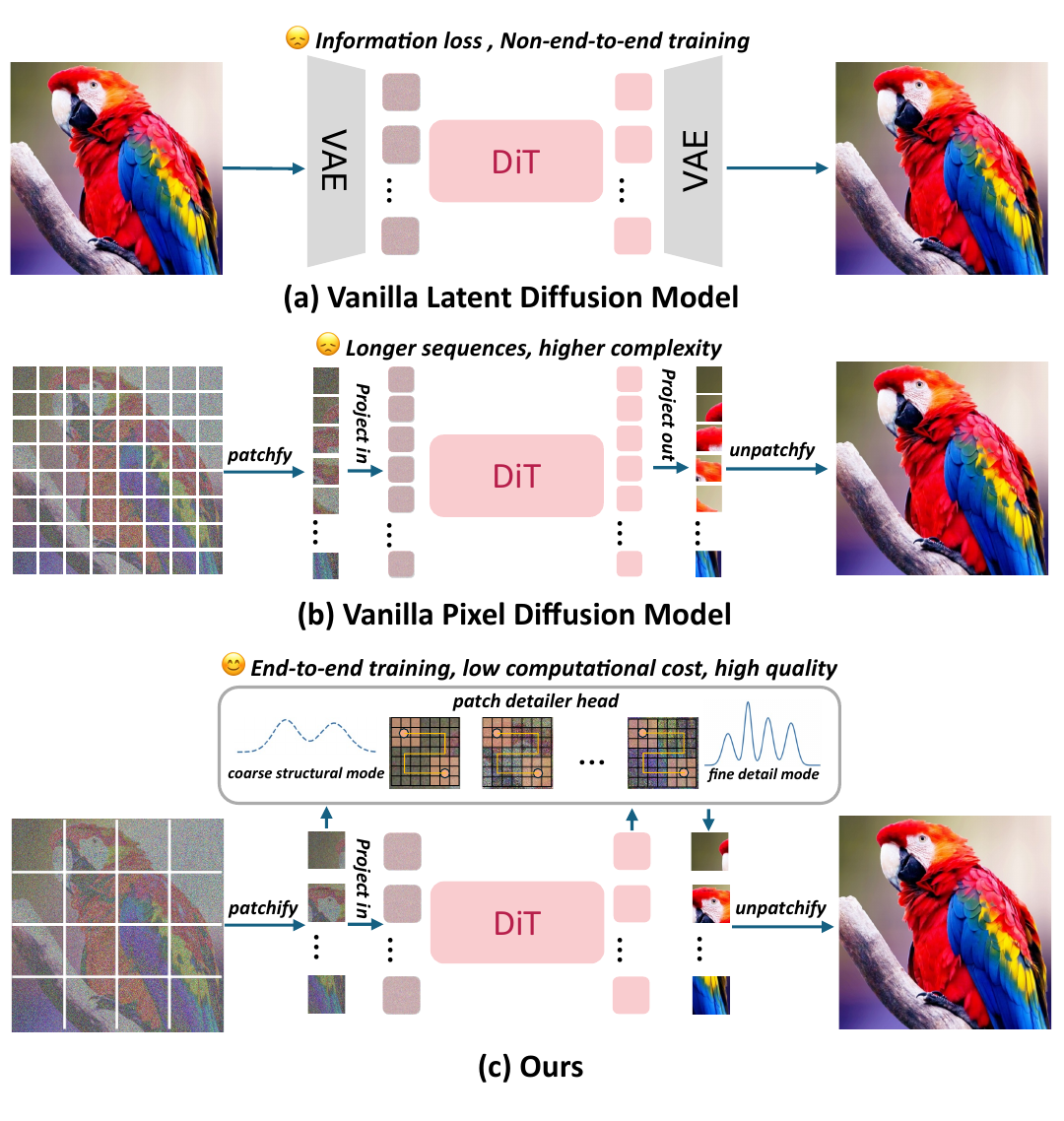}
    \vspace{-9mm}
    \caption{Comparison of vanilla latent diffusion model, vanilla pixel diffusion model and our method. Vanilla LDMs utilize VAEs to balance computational efficiency and generation quality. Vanilla pixel diffusion models use small patches to pursue detailed generation quality. Our method achieves high-quality generation while maintaining efficient end-to-end training in pixel space.}
    \label{Fig: teaser1}
    \vspace{-4mm}
\end{figure}

\vspace{-6mm}
\section{Introduction}

Diffusion models~\cite{sohl2015deep, ho2020denoising, song2020score, peebles2023scalable,ramesh2022hierarchical,saharia2022photorealistic,yu2022scaling, xie2024sana,song2020denoising,ho2022classifier,karras2024guiding} have reshaped the landscape of generative visual content. With its outstanding generative capabilities of fidelity and diversity, they have established new state-of-the-art benchmarks across a multitude of tasks, including image synthesis~\cite{ci2025describe,chen2023diffusion,ye2023ip,wang2024instantid,zhao2025ultrahr, chen2025ragd, zhou2024migc,zhou20243dis,zhao2024wavelet,chen2023pixart,du2025textcrafter}, video generation~\cite{fan2025instancecap,nan2024openvid,zhang2021rstnet,zhao2026luve}, and 3D object creation~\cite{zhang2024tar3d,zhang2025ar, zhang2024temo}, decisively surpassing prior paradigms like Generative Adversarial Networks (GANs)~\cite{zhao2024cycle,goodfellow2020generative,radford2015unsupervised,mirza2014conditional,zhu2017unpaired,karras2019style}. However, this generative prowess is underpinned by immense computational demands. Consequently, the inherent \textit{trade-off between generation quality and computational efficiency} thus stands as one of the most critical challenges in the field of diffusion models today.

To mitigate this challenge, Latent Diffusion Models (LDMs)~\cite{rombach2022high} have emerged as the de facto standard. By employing a pre-trained autoencoder (VAE)~\cite{kingma2013auto} to compress high-resolution images into a compact latent space, LDMs significantly reduce the computational complexity of the iterative denoising process, as shown in Figure~\ref{Fig: teaser1}(a). Nevertheless, this approach is not without its limitations, including \textit{potential information loss}~\cite{yao2025reconstruction,kilian2024computational,chen2024deep,gupta2024photorealistic} \textit{during VAE compression and a non-end-to-end training pipeline.}

\newcommand{\graydashedline}{%
  (\textcolor{gray}{\hdashrule{1.5em}{1.4pt}{2pt 1.3pt}})
}
\begin{figure}[t]
    \centering
    \includegraphics[width=\linewidth]{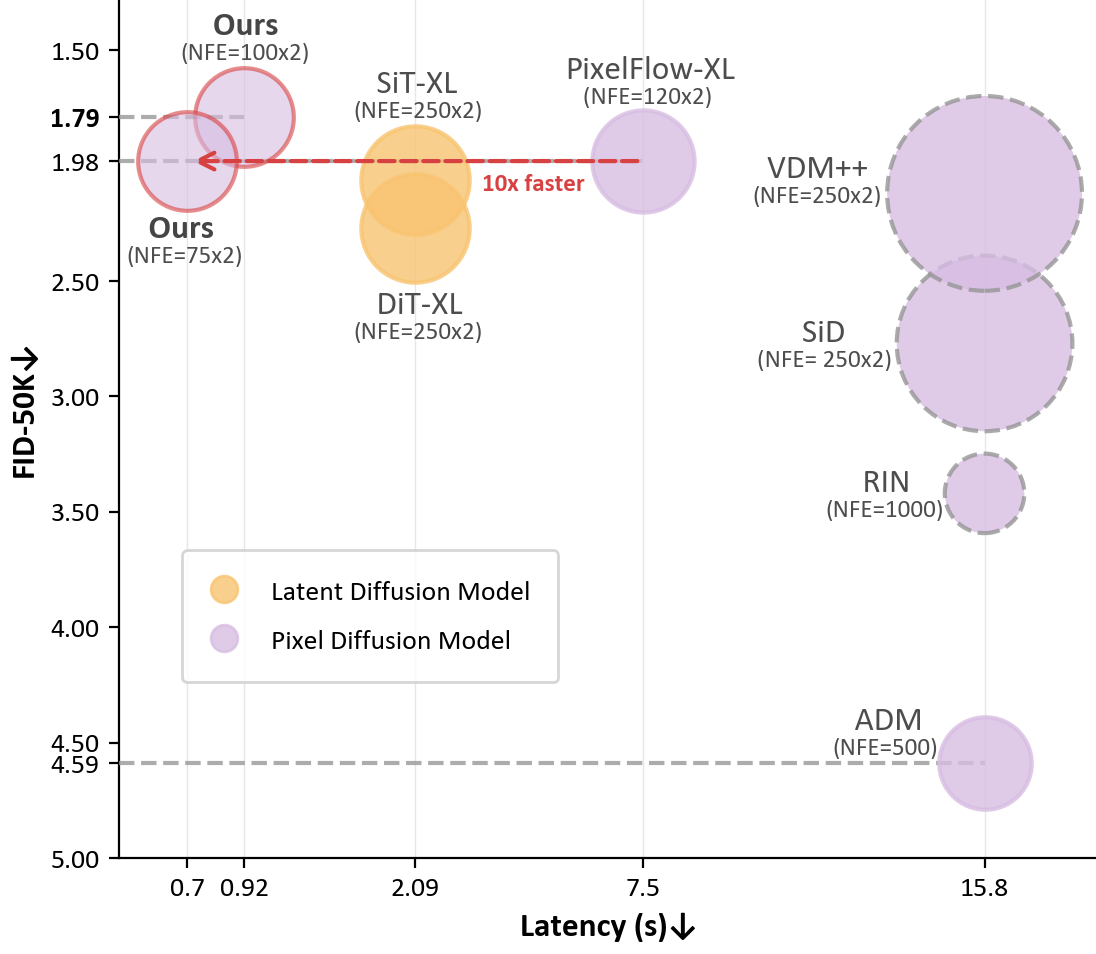}
    \vspace{-8mm}
    \caption{Our method achieves the best FID score with minimal computational cost. (Note: LDM latency includes VAE. The methods marked with dashed lines \graydashedline{} are our estimated latency based on the sampling method in the corresponding paper, and should actually be greater than the marked values. The rest methods are the actual test results in the same hardware environment.)}
    \label{Fig: teaser2}
    \vspace{-4mm}
\end{figure}

The most direct solution to eliminate the shortcomings of LDMs is to train a diffusion model in pixel space. However, existing pixel space diffusion models \cite{hoogeboom2024simpler,dhariwal2021diffusion,chen2025pixelflow,kingma2023understanding,hoogeboom2023simple,teng2023relay}, particularly those based on the powerful Transformer architecture~\cite{vaswani2017attention}, face a severe scalability issue. As shown in Figure~\ref{Fig: teaser1}(b), to capture fine-grained details, they typically rely on small input patches (\eg 2$\times$2 or 4$\times$4), causing the input sequence length to grow quadratically with image resolution. This quadratic scaling renders \textit{high-resolution training and inference computationally intractable}, creating a formidable barrier to their practical application.

In this paper, we aim to resolve this critical trade-off in pixel diffusion models. We propose an efficient pixel space diffusion framework called DiP. 
As shown in Figure~\ref{Fig: teaser1}(c), for efficient global structure construction, we employ a DiT~\cite{peebles2023scalable} backbone. Critically, we configure it to operate on large image patches (\eg, 16$\times$16). This setting choice drastically reduces the input sequence length, aligning it with that of mainstream LDMs operating in latent space. Consequently, our model achieves computational efficiency comparable to LDMs while remaining entirely VAE-free, enabling it to effectively capture the global layout and semantic content of the image. However, operating on large patches alone inevitably leads to blurry outputs lacking high-frequency details. To address this, we introduce a lightweight Patch Detailer Head (only 0.3\% increase in total parameters), which is not a post-processing module but an integral component co-trained with the DiT backbone. For each large patch, it receives contextual features from the DiT and leverages its strong local receptive fields to synthesize the missing high-frequency information. This synergistic design allows the DiT backbone to focus on the computationally demanding task of global consistency, while the efficient Patch Detailer Head specializes in local texture and detail restoration. As demonstrated in Figure~\ref{Fig: teaser2}, our approach sets a new state-of-the-art on the efficiency-quality frontier, achieving superior FID scores at significantly lower latency compared to existing methods.
Our main contributions are summarized as follows:
\begin{itemize}

    \item 
        We propose DiP, a new end-to-end pixel diffusion model framework that effectively alleviates the trade-off between generation quality and computational efficiency through synergistic global-local modeling.
    \item 
        We systematically validate the impact of different architectural designs of our framework, hoping to provide the community with a unified, principled framework.
    \item 
        On ImageNet generation benchmarks, our framework achieves state-of-the-art performance and lowest inference latency with low training costs.
\end{itemize}

\section{Related Work}

\noindent\textbf{Latent Diffusion Models.}
Latent Diffusion Models (LDMs)~\cite{rombach2022high,podell2023sdxl,esser2024scaling,flux,xie2024sana, chen2023pixart,zhang2023adding} have become the de-facto paradigm for large-scale generative modeling due to their computational efficiency and scalability. 
By performing the diffusion process in a compressed latent space learned by a VAE, LDMs drastically reduce memory and computational costs.
Architectural advancements within this paradigm, such as replacing the U-Net~\cite{ronneberger2015u} backbone with a more scalable Transformer (DiT)~\cite{peebles2023scalable}, have further pushed the boundaries of generation quality. Despite their success, this efficiency comes at a cost: the VAE acts as an information bottleneck, imposing a hard ceiling on the final image fidelity and often introducing subtle reconstruction artifacts~\cite{skorokhodov2025improving,yao2025reconstruction}. Our work circumvents these limitations by proposing an equally efficient architecture that operates directly in pixel space, thereby eliminating the VAE-induced quality constraints.

\noindent\textbf{Pixel Diffusion Models.}
Recent years have seen renewed interest in pixel space diffusion models that aim to maximize signal fidelity while addressing computational inefficiency. Early works such as ADM~\cite{dhariwal2021diffusion} and DDPM~\cite{ho2020denoising} demonstrated the power of diffusion but were constrained by the quadratic complexity of their backbones, rendering them impractical for high resolutions. Multi-scale and image patch-based methods~\cite{hoogeboom2024simpler,ding2023patched,chen2025pixelflow,hoogeboom2023simple,hoogeboom2024simpler} further enhance the generation effect by decomposing large images into small patches. However, these methods essentially simulate locality through brute-force training, which leads to extremely low efficiency. 
Concurrent work JiT~\cite{li2025jit} demonstrates that high-dimensional data in pixel space can be effectively modeled by predicting clean images.
Recent work by PixelNerd~\cite{wang2025pixnerd} leverages a Transformer to process image features, which then conditions a NeRF-like coordinate network to act as a renderer for finely reconstructing each image patch, achieving impressive performance. Nevertheless, PixelNerd tightly couples the success of its method with this specific NeRF-like rendering mechanism, which may limit the exploration of a broader design space. We argue that the key to achieving efficient and high-quality pixel space generation lies not in relying on a specific structure like NeRF, but rather in the design principle of decoupling global structure construction from local detail refinement. Based on this insight, this paper aims to provide a more principled, efficient, and general solution for pixel space diffusion models.

\section{Methods}
\subsection{Preliminaries}

A diffusion process gradually perturbs an initial data sample $\mathbf{x}_{0} \sim q\left(\mathbf{x}_{0}\right)$ from the true data distribution into isotropic Gaussian noise:
\begin{equation}
    \mathbf{x}_{t}=\sqrt{\bar{\alpha}}_{t} \mathbf{x}_{0}+\sqrt{1-\bar{\alpha}_{t}} \epsilon, \quad \text { where } \epsilon \sim \mathcal{N}(0, \mathbf{I}),
\end{equation}
where $\alpha_{t}=1-\beta_{t}$ and $\bar{\alpha}_{t}=\prod_{i=1}^{t} \alpha_{i}$. $\{\beta_{t}\}^{T}_{i=1}$ is a predefined variance schedule that controls the noise level at each step. As $t \rightarrow T, \bar{\alpha}_{t} \rightarrow 0$, and the distribution of $\mathbf{x}_{T}$ converges to a standard normal distribution $p\left(\mathbf{x}_{T}\right) \approx \mathcal{N}(0, \mathbf{I})$.

This discrete formulation can be generalized to a continuous-time setting via a stochastic differential equation (SDE):
\begin{equation}
    d \mathbf{x}=f(\mathbf{x}, t) d t+g(t) d \mathbf{w},
\end{equation}
where $f(\cdot, t)$ is the drift and $g(t)$ is the diffusion coefficient. 

The trajectory of this reverse process is governed by a corresponding probability flow ordinary differential equation (ODE):
\begin{equation}
    d \mathbf{x}=\left[f(\mathbf{x}, t)-g(t)^{2} \nabla_{\mathbf{x}} \log p_{t}(\mathbf{x})\right] d t.
\end{equation}

Learning to generate data is thus equivalent to learning the score function $\log p_{t}(\mathbf{x})$ or the associated vector field of this ODE. To train a neural network for this task, several objectives have been proposed. DDPM trains a model $\epsilon_{\theta}\left(\mathbf{x}_{t}, t\right)$ to predict the noise component $\theta$ from a noisy sample $x_{t}$:
\begin{equation}
    \mathcal{L}_{\mathrm{DDPM}}=\mathbb{E}_{t, \mathbf{x}_{0}, \epsilon}\left[\left\|\epsilon-\epsilon_{\theta}\left(\mathbf{x}_{t}, t\right)\right\|^{2}\right].
\end{equation}

Flow Matching (FM)~\cite{esser2024scaling} provides a simulation-free paradigm for directly learning the vector field. It defines a conditional probability path $p_{t}\left(\mathbf{x} \mid \mathbf{x}_{0}\right)$ and a corresponding target vector field $u_{t}(\mathbf{x})$. A network $v_{\theta}(\mathrm{x}, t)$ is then trained to regress this field by minimizing the loss:
\begin{equation}
    \mathcal{L}_{\mathrm{FM}}=\mathbb{E}_{t, p_{t}\left(\mathbf{x} \mid \mathbf{x}_{0}\right)}\left[\left\|u_{t}(\mathbf{x})-v_{\theta}(\mathbf{x}, t)\right\|^{2}\right].
\end{equation}

\begin{figure}[h]
    \centering
    \includegraphics[width=\linewidth]{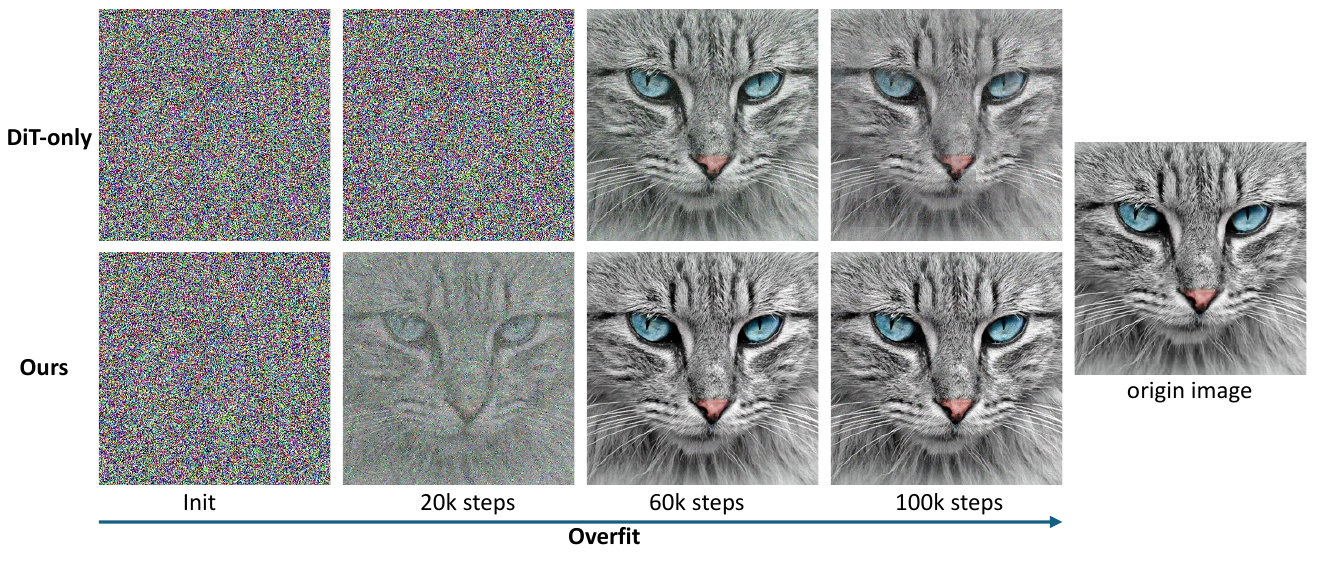}
    \vspace{-6mm}
    \caption{Overfitting the DiT-only model using a single image in pixel space leads to poor detail reconstruction. Introducing a local inductive bias achieves better reconstruction and accelerates convergence. Please zoom in for details.}
    \label{Fig:overfit}
\end{figure}

\begin{figure}[h]
    \centering
    \includegraphics[width=\linewidth]{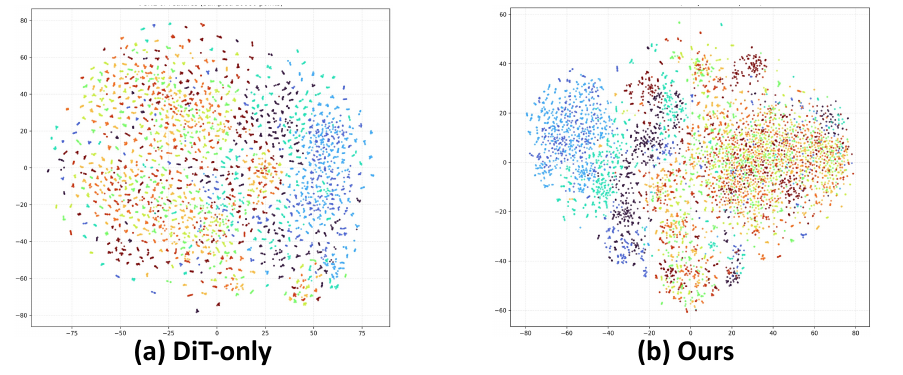}
    \vspace{-6mm}
    \caption{The t-SNE visualization of feature space. In the ImageNet validation set, 100 samples were randomly selected from each of the 10 classes for feature visualization. Features are extracted using DiT-only and our method, with each class shown in a distinct color.}
    \label{Fig:vis_feature_method}
    \vspace{-4mm}
\end{figure}

\begin{figure*}[t]
    \centering
    \includegraphics[width=\linewidth]{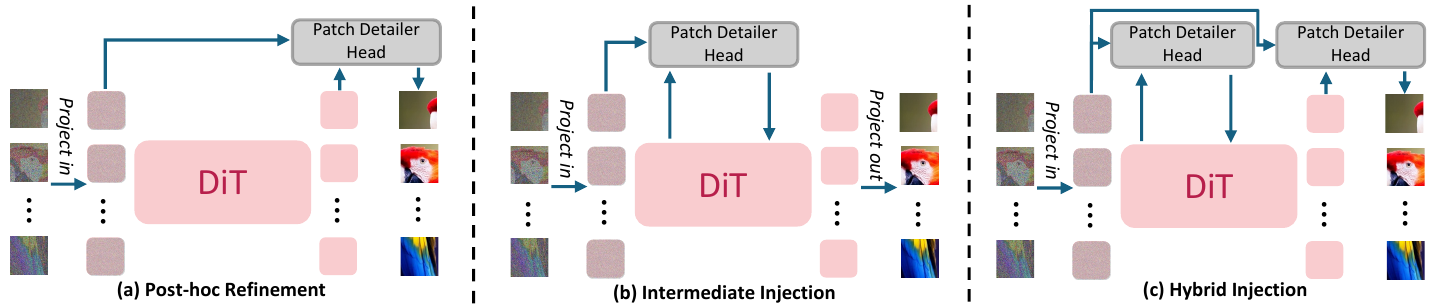}
    \vspace{-6mm}
    \caption{Patch Detailer Head with local inductive bias was placed at different locations in the model. The results in Sec.~\ref{Analysis} show that all three methods offer gains compared to DiT-only. 
    }
    \label{Fig:diff_insert}
    \vspace{-3mm}
\end{figure*}

\vspace{-4mm}
\subsection{Motivation}

DiT models the long-range dependencies of an image by partitioning it into a sequence of patches, thereby forming a coherent global structure. However, while the self-attention mechanism excels at modeling macroscopic relationships between patches, it compresses the rich spatial information within each patch into a single, flattened token. This design introduces an inherent limitation: the model can adeptly learn the coarse-level layout and arrangement of patches but struggles to model the fine-grained textures and high-frequency details within each patch, consequently limiting the upper bound of its image generation performance.

To empirically validate this, we conduct a preliminary experiment by overfitting a DiT model on a single high-resolution image in the pixel space. As shown in Figure~\ref{Fig:overfit}, the model successfully captures the global layout and color palette but fails to render fine textures and sharp edges, resulting in a blurry reconstruction. This result demonstrates that when a DiT architecture operates directly on images, it suffers from a lack of inductive bias~\cite{yang2024diffusion,goyal2022inductive,an2024inductive,kadkhodaie2023generalization} at the local level, rendering it incapable of achieving precise pixel-level reconstruction within each patch.

This motivates our core design principle: to augment the global Transformer with a dedicated module that explicitly re-injects this missing inductive bias for local details. In this way, our model can leverage the computational efficiency afforded by large patch sizes while simultaneously generating high-quality images with fine-grained details. As shown in  Figure~\ref{Fig:vis_feature_method}, our method achieves tighter intra-class clusters and clearer inter-class separation, whereas vanilla DiT exhibits more mixed distributions. This means that the introduction of local inductive bias can more effectively integrate local textures and edge cues in pixel space and thus improves high-level semantic separability and feature consistency. Such improvements are expected to yield more stable structural alignment and better detail during generation process.

\subsection{Framework}

Based on the above observation,  we introduce a framework for high-quality image generation that operates directly in pixel space. DiP first employs a DiT to model the global structure and long-range dependencies of the image. Subsequently, a lightweight Patch Detailer Head refines the output at the patch level, introducing a  local inductive bias to synthesize high-frequency details.

\noindent\textbf{Global Structure Construction (DiT Backbone).}
Given a noisy image $x_{t}\in \mathbb{R}^{H \times W \times 3}$ at timestep $t$,  we first partition it into a sequence of non-overlapping patches. Each patch has a size of $P$$\times$$P$ (we set P$=$16), resulting in a sequence of $N$$=$$(H \times W) / P^{2}$ patches. This patching strategy ensures our pixel space model maintains a computational footprint comparable to latent space DiT models. Along with a timestep embedding and positional embeddings, they are fed into a series of DiT blocks to produce a sequence of context-aware output features $S_{\text {global}} \in \mathbb{R}^{N \times D}$, where $D$ is the feature dimension.

\noindent\textbf{Local Detail Refinement (Patch Detailer Head).}
The Patch Detailer Head operates independently and in parallel on each patch. For each patch $i$, it takes two inputs: the corresponding global context map $s_{i}$ and the original noisy pixel patch $p_{i} \in \mathbb{R}^{3 \times P \times P}$, where $s_{i} \in \mathbb{R}^{D \times 1 \times 1}$ is obtained by reshaping and expanding $S_{\text{global}}$. Its objective is to leverage the global context from $S_{\text {global }}$ to accurately interpret the local noisy information in $p_{i}$, ultimately predicting the corresponding noise component $\epsilon_{i} \in \mathbb{R}^{3 \times P \times P}$ for that patch. After processing all $N$ patches in parallel, the resulting sequence of predicted noise patches $\left\{\epsilon_{i}\right\}_{i=1}^{N}$ is reassembled into a full-resolution noise prediction map.

\begin{figure}[t]
    \centering
    \includegraphics[width=\linewidth]{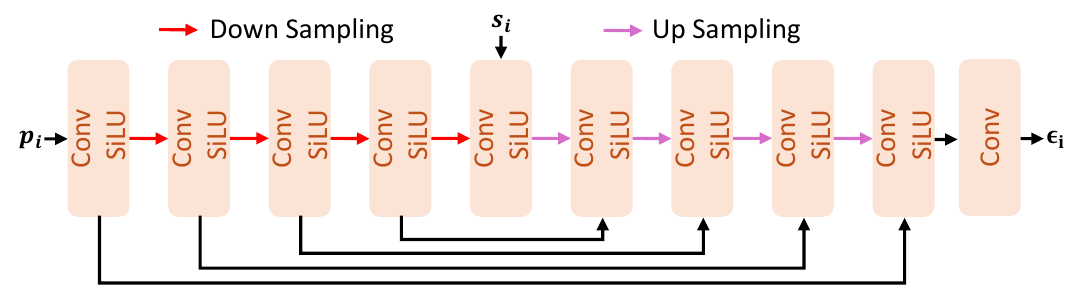}
    \vspace{-6mm}
    \caption{Patch Detailer Head framework. This design introduces the local inductive bias that DiT-only lacks with a low number of parameters, resulting in a high-quality image with rich detail.}
    \label{Fig:Patch_Enhancer}
    \vspace{-4mm}
\end{figure}

\begin{table*}[t]
	\centering
	\resizebox{\textwidth}{!}{
	\setlength{\tabcolsep}{3.8pt}
	\footnotesize
	\renewcommand\arraystretch{1.1}
\begin{tabular}{lccccccccc}
\toprule
\multicolumn{1}{c}{\multirow{2}{*}{\textbf{Method}}} 

& \multicolumn{9}{c}{\textbf{ ImageNet 256×256}} \\

\cmidrule(lr){2-10} 
& \textbf{FID}$\downarrow$
& \textbf{sFID}$\downarrow$ 
&\textbf{IS}$\uparrow$ 
& \textbf{Prec.}$\uparrow$ 
& \textbf{Rec.}$\uparrow$  
&\textbf{Latency}$\downarrow$
& \textbf{Epochs} 
& \textbf{NFE}
& \textbf{Params} \\

\midrule

\textbf{\textit{Latent Generative Models}} \\

\arrayrulecolor{gray!50}
\cmidrule{1-10} 
\arrayrulecolor{black} 
LDM~\cite{rombach2022high}        &3.60 &- &247.7 &0.87 &0.48 &- &170 &250x2  &400M+86M \\
DiT-XL~\cite{peebles2023scalable}  &2.27 &4.60 &278.2 &0.83 &0.57 &2.09s &1400 &250x2  &675M+86M\\
MaskDiT-G~\cite{zheng2023fast}  &2.28 &5.67 &276.6 &0.80 &0.61 &- &1600 &79x2  &675M+86M \\
SiT-XL~\cite{ma2024sit}     &2.06 &4.50 &270.3 &0.82 &0.59 &2.09s &1400 &250x2  &675M+86M\\
FlowDCN-XL~\cite{wang2024flowdcn} &2.00 &4.33 &263.1 &0.82 &0.58 &- &400 &250x2  &618M+86M \\ 

\midrule
\textbf{\textit{Pixel Generative Models}} \\
\arrayrulecolor{gray!50} 
\cmidrule{1-10}
\arrayrulecolor{black} 
CDM~\cite{ho2022cascaded}              &4.88 &- &158.7 &- &- &- &2160 &4100  &-  \\
ADM~\cite{dhariwal2021diffusion}              &3.94 &6.14 &215.8 &\textbf{0.83} &0.53 &15.80s &400 &500  &554M\\
JetFormer-L~\cite{tschannen2024jetformer}      &6.64 &- &- &0.69 &0.56 &- &500 &-   &2.8B\\
SiD~\cite{hoogeboom2023simple}              &2.77 &- &211.8 &- &-  &- &800 &250×2   &2.0B\\
VDM++~\cite{kingma2023understanding}            &2.12 &- &278.1 &- &- &- &- &250×2  &2.46B\\
RIN~\cite{jabri2022scalable}              &3.42 &- &182.0 &- &- &- &480 &1000  &410M  \\
Farmer/16~\cite{zheng2025farmer}        &3.96 &- &250.6 &0.79 &0.50 &- &320 &- &1.9B\\
PixelFlow-XL/4~\cite{chen2025pixelflow}   &1.98 &5.83 &282.1 &0.81 &0.60 &7.50s &320 &120x2 &677M  \\

\midrule
DiP-XL/16 &2.16 &4.79 &276.8 &0.82 &0.61 &0.92s &160 &100x2  &631M  \\
DiP-XL/16 &1.98 &\textbf{4.57} &\textbf{282.9} &0.80 &0.62 &\textbf{0.70s} &320 &75x2   &631M \\
DiP-XL/16 &\textbf{1.79} &4.59 &281.9 &0.80 &\textbf{0.63} &0.92s &600 &100x2  &631M \\

\bottomrule
\end{tabular}
}
\vspace{-2mm}
\caption{Comparison of the performance of different methods on ImageNet 256×256 with Euler solver and CFG. Performance metrics are annotated with $\uparrow$ (higher is better) and $\downarrow$ (lower is better). Our method achieves the best FID score. Furthermore, compared to other pixel diffusion models, we achieve the best performance across all metrics with the lowest latency. 
}
\label{Tab: Quantitative+exp}
\vspace{-3mm}
\end{table*}

\subsection{Architecture Design}
\noindent\textbf{Exploring Patch Detailer Head Architectures.}
We investigated several architectures for the Patch Detailer Head, each embodying a different form of inductive bias. Our goal is to present a simple, effective and highly efficient design.
\begin{itemize}
    \item \textit{Standard MLP.} As a simple baseline, we used MLP that takes the feature vector $s_{i}$ and a flattened noisy patch $p_{i}$ as input. While straightforward, this design lacks any inherent spatial bias, treating all pixels within the patch as an unordered set.
    
    \item \textit{Coordinate-based MLP.} To introduce spatial awareness, a design inspired by NeRF can be adopted \cite{wang2025pixnerd}. For each pixel within $p_{i}$, we concatenate its normalized 2D coordinates. $s_{i}$ is used to dynamically generate the weights of a small, coordinate-based MLP. This implicitly learns a continuous function of the image patch, but it lacks the strong priors for local texture and structure that convolutions provide.
    
    \item \textit{Intra-Patch Attention.} We explored using a small Transformer to operate on the pixels within each patch. Each $P$$\times$$P$ patch is treated as a sequence of $P^{2}$ pixel tokens. This allows for complex, content-aware interactions between pixels but is computationally intensive and may not be as efficient as convolutions for learning local patterns.

    \item \textit{Convolutional U-Net (Our Final Choice).}  We found that a lightweight convolutional U-Net provided the best performance. The hierarchical structure of downsampling and upsampling paths, combined with skip connections, is exceptionally well-suited for capturing multi-scale spatial features and ensuring local continuity. The inherent inductive biases of convolutions (locality and translation equivariance) are highly effective for denoising local textures and edges. As shown in Figure~\ref{Fig:Patch_Enhancer}, we instantiate the Patch Detailer Head with a shallow U-Net, which includes 4 downsampling and 4 upsampling blocks.  Each block consists of a sequence of Convolution, SiLU activation and pooling layer. The global feature vector $s_{i}$ is concatenated channel-wise with the downsampling output at the bottleneck. This design allows the global semantic information to guide the local refinement process effectively while keeping the parameter count minimal.
      
\end{itemize}

We provide experimental evidence for this part in Sec.~\ref{Analysis}.
Furthermore, we present a preliminary theoretical analysis in Appendix to model the necessity and effectiveness of the Patch Detailer Head, aiming to offer deeper insights.

\begin{figure*}[t]
    \centering
    \includegraphics[width=\linewidth]{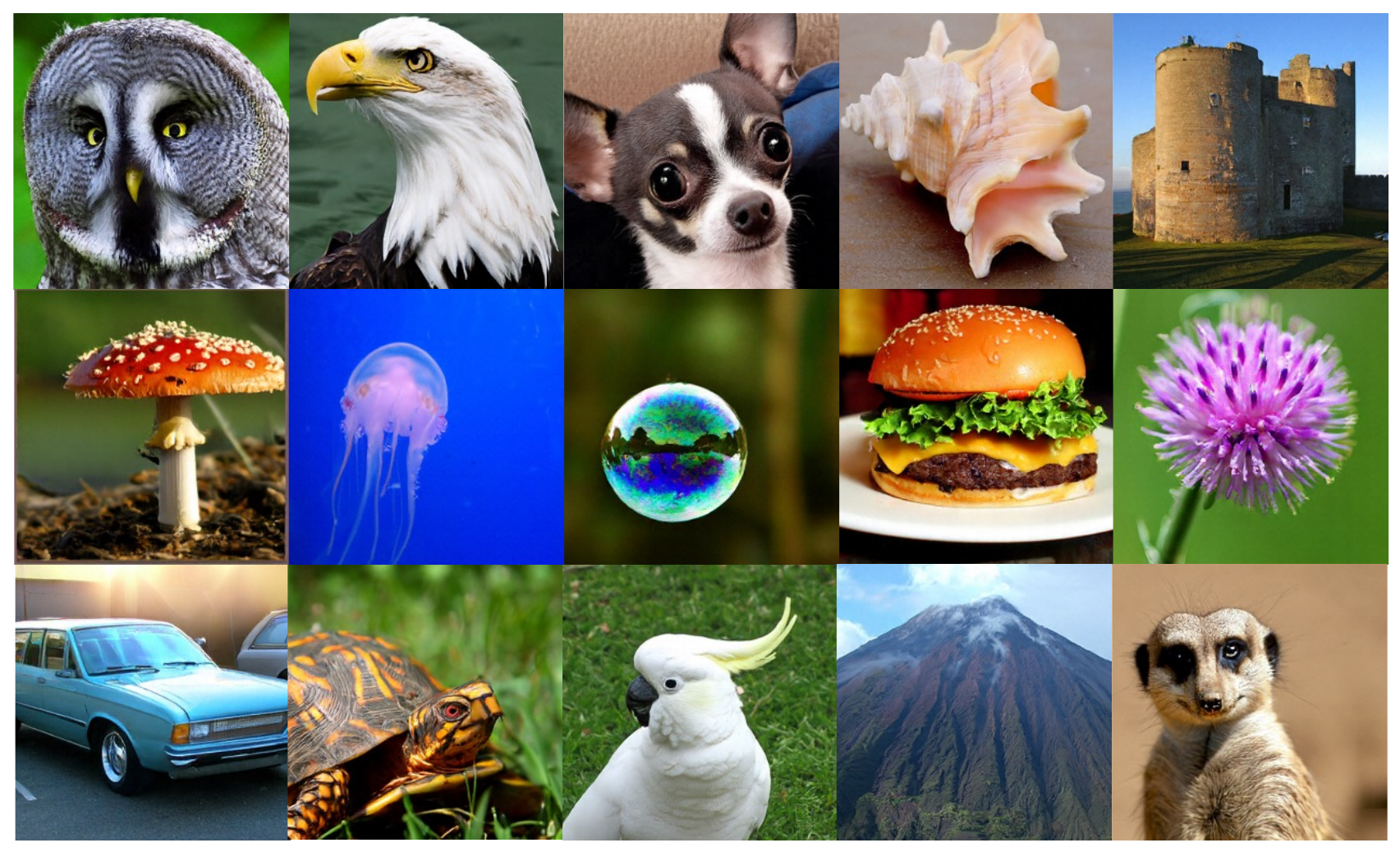}
    \vspace{-6mm}
    \caption{Qualitative samples from our model trained at 256 × 256 resolution with classifier-free guidance scale of 4.0. DiP demonstrates fine-grained detail, and high visual quality.}
    \label{Fig:vis_256}
    \vspace{-3mm}
\end{figure*}

\noindent\textbf{Placement of the Patch Detailer Head.}
Since the main weakness of DiT backbone being trained directly in pixel space is its lack of local awareness, a natural design question arises: does introducing Patch Detailer Head at different locations in the model also bring gains? As shown in Figure~\ref{Fig:diff_insert}, we investigated three placement strategies:

\begin{itemize}
    \item \textit{Post-hoc Refinement.} The Patch Detailer Head is placed only after the final DiT block. This creates a clean separation of concerns: the DiT is solely responsible for global modeling, and Patch Detailer Head is solely responsible for local refinement.
    
    \item \textit{Intermediate Injection.} The Patch Detailer Head is inserted between DiT blocks. The refined patch representations are then projected back and fed into the subsequent DiT blocks.

    \item \textit{Hybrid Injection.} Patch Detailer Head are placed both at an intermediate stage and at the end of the DiT.
    
\end{itemize}

Our experiments in Sec.~\ref{Analysis} revealed that all three strategies yield comparable performance gains over the baseline DiT. However, the Post-hoc Refinement strategy has a unique advantage: by placing the Head at the end, we treat the standard DiT architecture as a fixed, black-box backbone. This approach requires no modification to the DiT's internal structure, greatly simplifying implementation and potentially allowing for the use of pre-trained DiT checkpoints. Given its optimal balance between high performance and implementation simplicity, we adopt the post-refinement strategy as the final architecture.

\vspace{-2mm}
\section{Experiments}

\subsection{Setup}
\noindent\textbf{Implementation Details.}
Our experiments are conducted on the class-conditional ImageNet dataset and original images are center-cropped and resized to $256 \times 256$ resolution. We set global batch size to 256. We use DDT~\cite{wang2025ddt}, a variant of DiT, as our model backbone and apply an Exponential Moving Average (EMA) on the model weights with a decay factor of 0.9999. In Patch Detailer Head, the kernel size of the middle layers is set to 3, the padding to 1, and the kernel size of the last convolutional layer is set to 1. Unless otherwise specified, all samples were generated using the Euler-100 solver. More details are included in Appendix.

\noindent\textbf{Evaluation Protocol.}
To ensure a comprehensive and rigorous assessment of our model's generative capabilities, we adhere to the evaluation protocol established by ADM~\cite{dhariwal2021diffusion}. We employ a suite of standard quantitative metrics to measure performance across different dimensions. Specifically, we use the Fréchet Inception Distance (FID)~\cite{heusel2017gans} to assess overall realism and fidelity, the Spatial FID (sFID)~\cite{nash2021generating} to evaluate spatial and structural coherence, and the Inception Score (IS)~\cite{salimans2016improved} to measure class-conditional diversity. Furthermore, we report Precision (Prec.)/Recall (Rec.)~\cite{kynkaanniemi2019improved} to respectively quantify the fidelity of individual samples and the model's ability to cover the true data distribution. All metrics are calculated using 50,000 generated samples.

\begin{table*}[t]
	\centering
	\resizebox{\textwidth}{!}{
	\setlength{\tabcolsep}{3.8pt}
	\footnotesize
	\renewcommand\arraystretch{1.1}
\begin{tabular}{lcccccccc}
\toprule
\multicolumn{1}{c}{\multirow{2}{*}{\textbf{Method}}} 

& \multicolumn{8}{c}{\textbf{ ImageNet 256×256}} \\

\cmidrule(lr){2-9} 
& \textbf{FID}$\downarrow$
& \textbf{sFID}$\downarrow$ 
&\textbf{IS}$\uparrow$ 
& \textbf{Prec.}$\uparrow$ 
& \textbf{Rec.}$\uparrow$  
&\textbf{Training Cost} 
& \textbf{Latency}
& \textbf{Params} \\

\midrule

\textbf{\textit{Scaling Up DiT}} \\

\arrayrulecolor{gray!50} 
\cmidrule{1-9} 
\arrayrulecolor{black} 
DiT-only (26 Layers, 1152 Hidden Dim)  &5.28 &6.56 &243.8 &0.74 &0.55 &84$\times$8 GPU Hours &0.88s &629M \\
DiT-only (32 Layers, 1152 Hidden Dim)  &4.91 &6.44 &251.7 &0.74 &0.56 &103$\times$8 GPU Hours &1.05s &772M \\
DiT-only (26 Layers, 1280 Hidden Dim)  &4.28 &6.26 &249.6 &0.77 &0.56 &103$\times$8 GPU Hours &1.06s &776M \\
DiT-only (26 Layers, 1536 Hidden Dim)  &2.83 &5.16 &\textbf{285.6} &0.80 &0.57 &149$\times$8 GPU Hours &1.49s &1.1B \\

\midrule
\textbf{\textit{Different Patch Detailer Head}} \\
\arrayrulecolor{gray!50}
\cmidrule{1-9} 
\arrayrulecolor{black} 
Standard MLP              &6.92 &7.27 &210.9 &0.79 &0.41 &93$\times$8 GPU Hours &0.91s &630M\\
Intra-Patch Attention     &2.98 &5.16 &275.0 &0.80 &0.56 &96$\times$8 GPU Hours &0.94s &630M\\
Coordinate-based MLP      &2.20 &4.49 &284.6 &0.80 &0.58 &123$\times$8 GPU Hours &0.95s &700M\\
Convolutional U-Net      &\textbf{2.16} &\textbf{4.79} &276.8 &\textbf{0.82} &\textbf{0.61} &92$\times$8 GPU Hours &0.92s &631M\\

\bottomrule
\end{tabular}
}
\vspace{-2mm}
\caption{Impact of different design schemes on computational overhead and performance.}
\label{Tab: Analysis}
\vspace{-3mm}
\end{table*}

\subsection{Main Result}
\noindent\textbf{Performance.}
Table~\ref{Tab: Quantitative+exp} presents a comprehensive comparison against recent SOTA methods with classifier-free guidance scheduling with guidance interval~\cite{kynkaanniemi2024applying}. After 600 training epochs, DiP achieved an FID of 1.79 without requiring a pre-trained VAE, surpassing potentially diffusion models such as DiT-XL (FID 2.27) and SiT-XL (FID 2.06), which require longer training times. DiP outperforms the previous best pixel-based model, PixelFlow-XL/4 (FID 1.98), and significantly exceeds others like ADM (FID 3.94) and VDM++ (FID 2.12). Even with a shorter training schedule of 160 epochs, our model reaches a competitive FID of 2.16, outperforming established models like DiT-XL that require much longer training.

Figure~\ref{Fig:vis_256} presents qualitative samples of DiP at $256 \times 256$ resolution. These visualizations reveal rich detail, demonstrating the effectiveness of introducing local inductive bias. More visualization samples are provided in Appendix.

\noindent\textbf{Computational Cost Comparison.}
DiP's parameter count (631M) is significantly smaller than other pixel models, such as VDM++ (2.0B) and Farmer (1.9B). DiP reaches its best performance with only 320 epochs, which is over 4$\times$ more efficient than DiT-XL and SiT-XL (1400 epochs) and substantially faster than many other pixel-based methods like CDM (2160 epochs). 
In single-image inference speed tests, DiP (0.92s) is more than 2.2$\times$ faster than DiT-XL (2.09s) and more than 8$\times$ faster than the previous best pixel model, PixelFlow-XL (7.50s). Furthermore, in 75-step inference, DiP (0.70s) achieved the same FID score as PixelFlow-XL with a speed more than 10$\times$ faster.

\subsection{Analysis}
\label{Analysis}

In this section, we analyzed the trade-off between generation quality and computational cost during the development of DiP, and at the same time explained the rationality of the Patch Detailer Head we designed.

\noindent\textbf{Patch Detailer Head vs. Scaling Up DiT.}
A common strategy to improve generative models is to increase the model size. However, our findings indicate that this is a suboptimal approach for pixel space diffusion models. As shown in Table~\ref{Tab: Analysis}, increasing the DiT's depth from 26 to 32 layers yields only a marginal improvement (FID from 5.28 to 4.91) at a considerable cost in parameters and training time. It also means that the effectiveness of our Patch Detailer Head comes from the introduction of effective local inductive biases, rather than increasing network depth.

In contrast, widening the model proves more effective for quality improvement. For instance, scaling the hidden dimension to 1536 reduces the FID to 2.83. This substantial quality gain comes at a prohibitive cost: a 74.9\% increase in parameters (from 629M to 1.1B), a 77.4\% rise in training cost (from 84$\times$8 to 149$\times$8 GPU hours), and a 69.3\% (from 0.88s to 1.49s) increase in inference latency. This highlights a critical challenge with monolithic scaling, where significant computational resources are required for performance.

\begin{figure}[t]
    \centering
    \includegraphics[width=\linewidth]{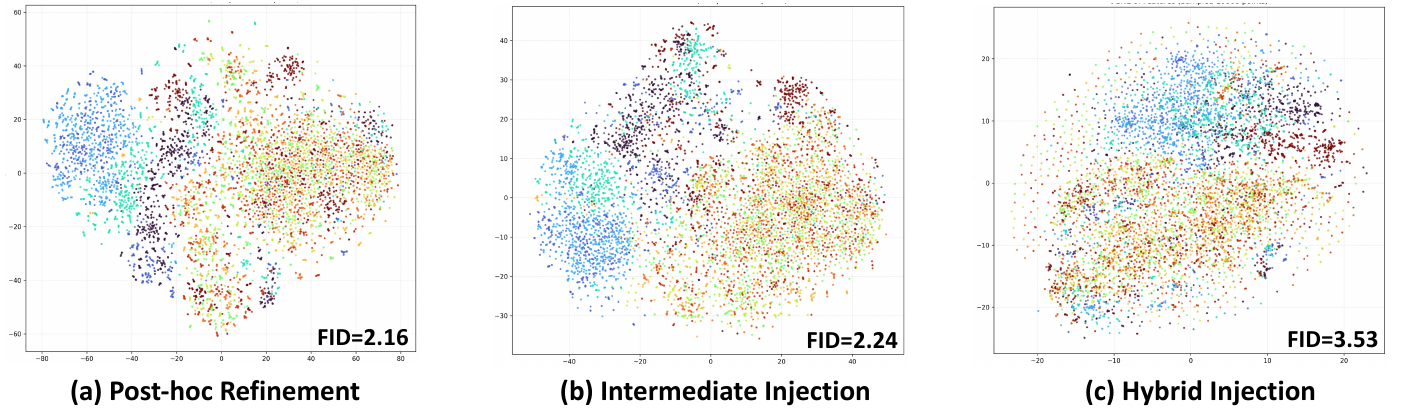}
    \vspace{-6mm}
    \caption{The t-SNE visualization of feature space. Features are extracted using Post-hoc Refinement, Intermediate Injection, and Hybrid Injection, with each class shown in a distinct color.}
    \label{Fig:vis_feature_place}
    \vspace{-5mm}
\end{figure}
\begin{table}[t]
	\centering
	\resizebox{\linewidth}{!}{
	\footnotesize
	\renewcommand\arraystretch{1.1}
\begin{tabular}{lcccccc}
\toprule
\multicolumn{1}{c}{\multirow{2}{*}{\textbf{Method}}} 

& \multicolumn{6}{c}{\textbf{ ImageNet 512×512}} \\

\cmidrule(lr){2-7} 
& \textbf{FID}$\downarrow$ 
& \textbf{sFID}$\downarrow$ 
& \textbf{Prec.}$\uparrow$
& \textbf{Rec.}$\uparrow$
&\textbf{IS}$\uparrow$
& \textbf{Params} \\

\midrule

\textbf{\textit{Latent Generative Models}} \\

\arrayrulecolor{gray!50} 
\cmidrule{1-7} 
\arrayrulecolor{black} 
DiT-XL~\cite{peebles2023scalable}     &3.04  &5.02  &\textbf{0.84}  &0.54  &240.8 &675M+86M \\
MaskDiT-G~\cite{zheng2023fast}  &2.50  &5.10  &0.83  &0.56  &256.3 &675M+86M \\
SiT-XL~\cite{ma2024sit}     &2.62  &\textbf{4.18}  &\textbf{0.84}  &0.57  &252.2 &675M+86M \\
FlowDCN-XL~\cite{wang2024flowdcn} &2.44  &4.53  &\textbf{0.84}  &0.54  &252.8 &618M+86M \\ 

\midrule
\textbf{\textit{Pixel Generative Models}} \\
\arrayrulecolor{gray!100} 
\cmidrule{1-7}
\arrayrulecolor{black} 
ADM~\cite{dhariwal2021diffusion}       &3.85  &5.86  &\textbf{0.84}  &0.53  &221.7 &554M \\
SiD~\cite{hoogeboom2023simple}       &3.02  &- &- &- &248.7 &2.00B \\
VDM++~\cite{kingma2023understanding}     &2.65  &- &- &- &278.1 &2.46B \\
RIN~\cite{jabri2022scalable}        &3.95  &- &- &- &216.0 &410M \\
\midrule
DiP-XL/32 &\textbf{2.31} &4.48 &\textbf{0.84} &\textbf{0.58} &\textbf{291.68} &631M \\

\bottomrule
\end{tabular}
}
\vspace{-2mm}
\caption{Comparison of the performance of different methods on ImageNet 512×512 with CFG. Performance metrics are annotated with $\uparrow$ (higher is better) and $\downarrow$ (lower is better). Our method remains competitive at higher resolutions.}
\label{Tabl: 512_exp}
\vspace{-5mm}
\end{table}

\noindent\textbf{Experimental Results of Exploring Patch Detailer Head Architectures.}
We further investigated different architectures for the Patch Detailer Head to understand the importance of inductive bias in local patch refinement. (1) The Standard MLP performs poorly (6.92 FID), even worse than the DiT-only baseline. This is expected, as it lacks any spatial inductive bias, treating patch pixels as an unordered set and failing to capture crucial local structures. (2) The Intra-Patch Attention shows a significant improvement over the MLP (FID 2.98). This indicates that content-aware relationships between pixels are valuable. Its training and inference costs are only slightly higher than our final choice, but its actual memory overhead is about twice that of the final solution. (3) The Coordinate-based MLP achieves 2.20 FID (we are based on a reproduction of~\cite{wang2025pixnerd}). By explicitly conditioning on pixel coordinates, it effectively introduces spatial awareness. However, it requires more parameters (700M) and a longer training time (123$\times$8 GPU hours) compared to our final choice, and its implicit continuous representation may lack the strong, built-in priors for local patterns that convolutions provide. (4) The Convolutional U-Net increases the number of parameters by only 0.3\% (from 629M to 631M) and achieves the best FID score with the lowest computational cost among all Patch Detailer Head. Its success can be attributed to highly relevant inductive biases of convolutions. It is well-suited for capturing and preserving the continuity of local textures and edges, making it the most efficient and effective architecture for performing patch-level detail optimizations.

In summary, our experimental results clearly demonstrates that introducing an appropriate local inductive bias via a Patch Detailer Head is key to performance improvement over the scaling up DiT baseline. Among the architectures explored, the Convolutional U-Net strikes the optimal balance between best generation quality and minimal computational cost, making it our definitive choice. 

\begin{figure}[t]
    \centering
    \includegraphics[width=\linewidth]{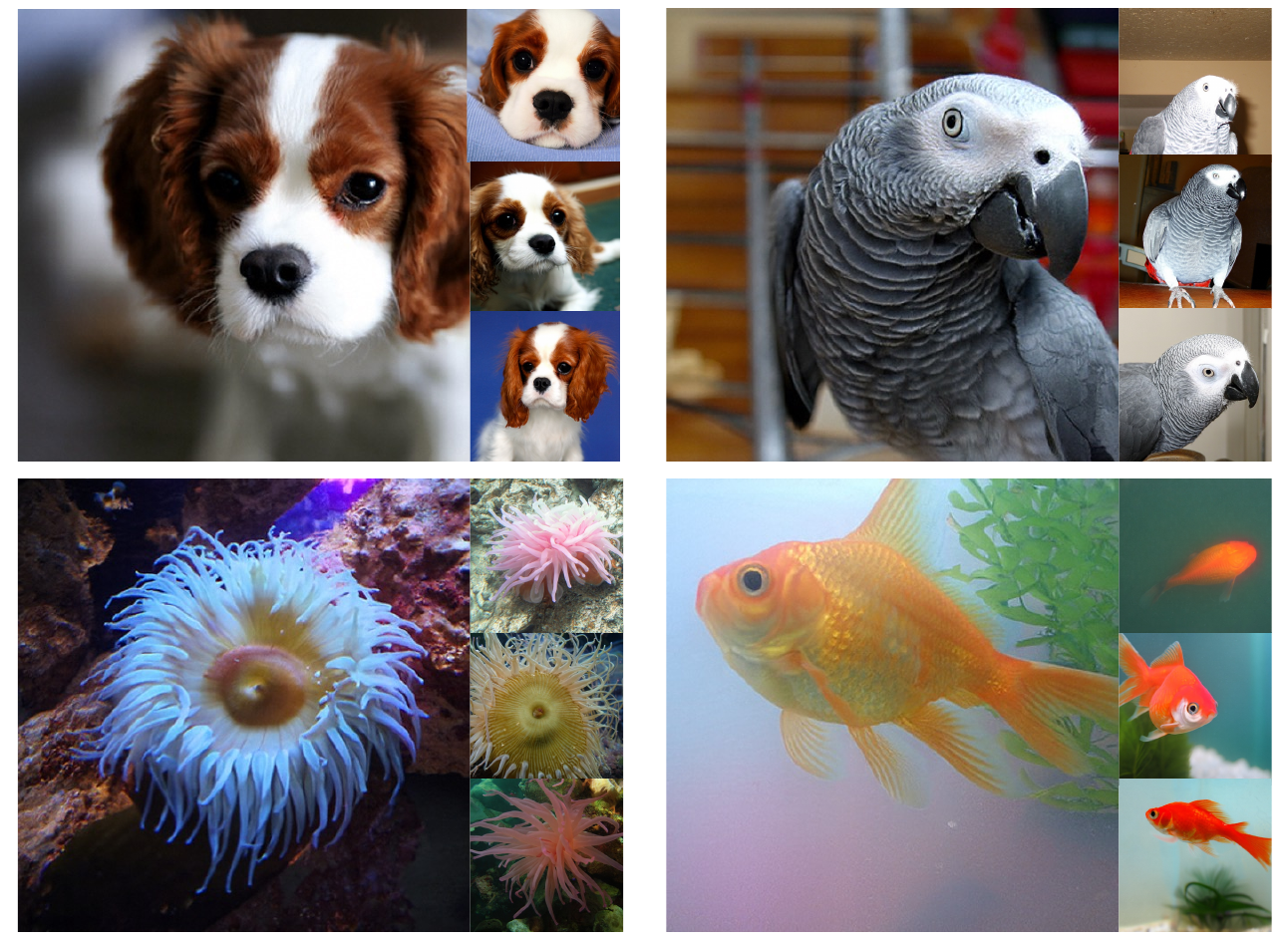}
    \vspace{-6mm}
    \caption{Qualitative samples from our model trained at 512$\times$512 resolution with classifier-free guidance scale of 4.0. DiP showcases fine-grained detail and rich diversity at higher resolutions.}
    \label{Fig:vis_512}
    \vspace{-4mm}
\end{figure}

\noindent\textbf{Experimental Results of Placement of the Patch Detailer Head.}
We tested the effects of introducing Patch Detailer Head at different locations in the model. As shown in Figure~\ref{Fig:vis_feature_place}, all three introduction modes showed significant improvements compared to DiT-only (FID 5.28). From the feature visualization results, Hybrid Injection performed worse in clustering than the other two, which may be due to multiple local inductive biases potentially disrupting the original structure, leading to performance degradation. Post-hoc Refinement achieves the best performance, a result  attributed to the synergy between global build and local refinement, while its implementation is simple and easily extensible.

\begin{figure}[t]
    \centering
    \includegraphics[width=\linewidth]{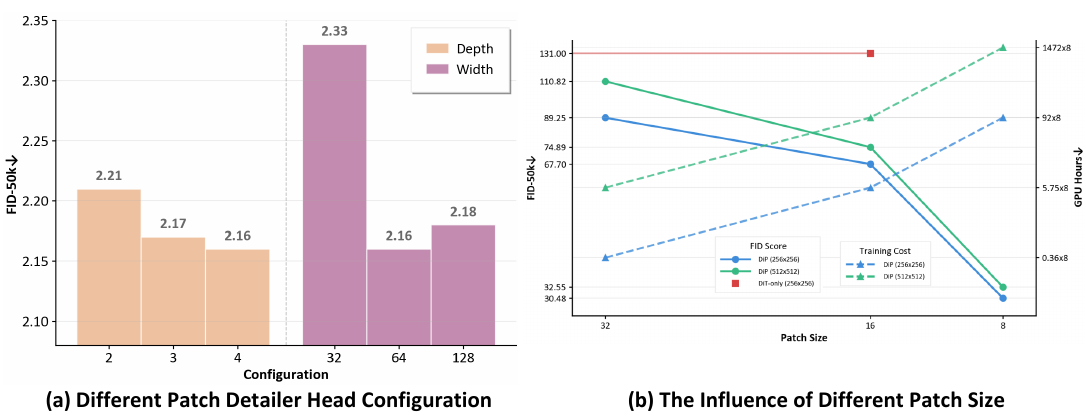}
    \vspace{-6mm}
    \caption{(a). Performance differences between different Patch Detailer Head configurations. Depth is defined as the number of down/up-sampling stages, and width corresponds to the number of base channels in the convolutional layers. (b). Performance and computational overhead differences of different patch sizes.}
    \label{Fig: ablation}
    \vspace{-4mm}
\end{figure}

\subsection{Ablation Study}
\noindent\textbf{Performance on ImageNet 512$\times$512.}
As shown in Table~\ref{Tabl: 512_exp}, on 512$\times$512 resolution, DiP achieved the best FID score , surpassing previous methods. Figure~\ref{Fig:vis_512} illustrates the sampling results at 512$\times$512, demonstrating that DiP can also generate high-quality images. We present more qualitative results in Appendix.

\noindent\textbf{Impact of Patch Detailer Head Configuration.}
Our findings reveal distinct trends for depth and width. As we increase the depth, we observe a consistent improvement in generation quality, although the gains exhibit diminishing returns and eventually saturate. This suggests that a multi-scale feature hierarchy is crucial for the Patch Detailer Head to effectively synthesize high-frequency details across. Conversely, blindly increasing the width will not lead to a sustained performance improvement. This indicates that the role of Patch Detailer Head is not to perform complex feature transformations, but rather to render specific details.

\noindent\textbf{Impact of Patch Size.}
Small patch size offer superior performance but also significantly increase computational overhead. 
Our method can use large patch to shorten the input sequence length, making our model's computational efficiency comparable to mainstream LDMs. As shown in Figure~\ref{Fig: ablation}(b), at the higher 512$\times$512 resolution, DiP maintains a significant performance margin over a DiT-only baseline using smaller patch size. This validates that our approach provides a robust solution for efficient, high-resolution synthesis directly in pixel space.

\vspace{-2mm}
\section{Conclusion}
In this paper, we addressed the fundamental trade-off between generation quality and computational efficiency in pixel diffusion models. We introduced DiP, a new end-to-end pixel space diffusion framework that resolves this dilemma through a synergistic global-local modeling approach. By employing a DiT backbone on large image patches, we achieve computational efficiency comparable to LDMs for modeling global structures. This is complemented by a co-trained, lightweight Patch Detailer Head that expertly restores high-frequency details, effectively bypassing the need for a VAE. 
Our extensive experiments on the ImageNet benchmark demonstrate that achieves superior FID scores with significantly lower inference latency and training costs. In the future, we plan to apply the DiP framework to text to image and text to video tasks to further explore the capabilities of this solution.
\section*{Acknowledgment}
This work was supported by Natural Science Foundation of Jiangsu Province: BK20241198, the Gusu Innovation and Entrepreneur Leading Talents: No. ZXL2024362 and Natural Science Foundation of China: No. 62406135.

{
    \small
    \bibliographystyle{ieeenat_fullname}
    \bibliography{main}

\begin{thebibliography}{76}
\providecommand{\natexlab}[1]{#1}
\providecommand{\url}[1]{\texttt{#1}}
\expandafter\ifx\csname urlstyle\endcsname\relax
  \providecommand{\doi}[1]{doi: #1}\else
  \providecommand{\doi}{doi: \begingroup \urlstyle{rm}\Url}\fi

\bibitem[An et~al.(2024)An, Wang, Guo, Luo, and Schwing]{an2024inductive}
Jie An, De Wang, Pengsheng Guo, Jiebo Luo, and Alex Schwing.
\newblock On inductive biases that enable generalization in diffusion transformers.
\newblock In \emph{The Thirty-ninth Annual Conference on Neural Information Processing Systems}, 2024.

\bibitem[BlackForest(2024)]{flux}
BlackForest.
\newblock Black forest labs; frontier ai lab, 2024.

\bibitem[Chen et~al.(2023{\natexlab{a}})Chen, Yu, Ge, Yao, Xie, Wu, Wang, Kwok, Luo, Lu, et~al.]{chen2023pixart}
Junsong Chen, Jincheng Yu, Chongjian Ge, Lewei Yao, Enze Xie, Yue Wu, Zhongdao Wang, James Kwok, Ping Luo, Huchuan Lu, et~al.
\newblock Pixart-$\alpha$: Fast training of diffusion transformer for photorealistic text-to-image synthesis.
\newblock \emph{arXiv preprint arXiv:2310.00426}, 2023{\natexlab{a}}.

\bibitem[Chen et~al.(2024)Chen, Cai, Chen, Xie, Yang, Tang, Li, Lu, and Han]{chen2024deep}
Junyu Chen, Han Cai, Junsong Chen, Enze Xie, Shang Yang, Haotian Tang, Muyang Li, Yao Lu, and Song Han.
\newblock Deep compression autoencoder for efficient high-resolution diffusion models.
\newblock \emph{arXiv preprint arXiv:2410.10733}, 2024.

\bibitem[Chen et~al.(2025{\natexlab{a}})Chen, Ge, Zhang, Sun, and Luo]{chen2025pixelflow}
Shoufa Chen, Chongjian Ge, Shilong Zhang, Peize Sun, and Ping Luo.
\newblock Pixelflow: Pixel-space generative models with flow.
\newblock \emph{arXiv preprint arXiv:2504.07963}, 2025{\natexlab{a}}.

\bibitem[Chen et~al.(2023{\natexlab{b}})Chen, Gao, Xiang, and Lin]{chen2023diffusion}
Zhennan Chen, Rongrong Gao, Tian-Zhu Xiang, and Fan Lin.
\newblock Diffusion model for camouflaged object detection.
\newblock In \emph{ECAI 2023}, pages 445--452. IOS Press, 2023{\natexlab{b}}.

\bibitem[Chen et~al.(2025{\natexlab{b}})Chen, Li, Wang, Chen, Jiang, Li, Wang, Yang, and Tai]{chen2025ragd}
Zhennan Chen, Yajie Li, Haofan Wang, Zhibo Chen, Zhengkai Jiang, Jun Li, Qian Wang, Jian Yang, and Ying Tai.
\newblock Ragd: Regional-aware diffusion model for text-to-image generation.
\newblock In \emph{Proceedings of the IEEE/CVF International Conference on Computer Vision}, pages 19331--19341, 2025{\natexlab{b}}.

\bibitem[Ci et~al.(2025)Ci, Guan, Ge, Zhang, Li, Zhang, Yang, and Tai]{ci2025describe}
En Ci, Shanyan Guan, Yanhao Ge, Yilin Zhang, Wei Li, Zhenyu Zhang, Jian Yang, and Ying Tai.
\newblock Describe, don't dictate: Semantic image editing with natural language intent.
\newblock In \emph{Proceedings of the IEEE/CVF International Conference on Computer Vision}, pages 19185--19194, 2025.

\bibitem[Dhariwal and Nichol(2021)]{dhariwal2021diffusion}
Prafulla Dhariwal and Alexander Nichol.
\newblock Diffusion models beat gans on image synthesis.
\newblock \emph{Advances in neural information processing systems}, 34:\penalty0 8780--8794, 2021.

\bibitem[Ding et~al.(2023)Ding, Zhang, Wu, and Tu]{ding2023patched}
Zheng Ding, Mengqi Zhang, Jiajun Wu, and Zhuowen Tu.
\newblock Patched denoising diffusion models for high-resolution image synthesis.
\newblock In \emph{The twelfth international conference on learning representations}, 2023.

\bibitem[Du et~al.(2025)Du, Chen, Gao, Chen, Chen, Jiang, Yang, and Tai]{du2025textcrafter}
Nikai Du, Zhennan Chen, Shan Gao, Zhizhou Chen, Xi Chen, Zhengkai Jiang, Jian Yang, and Ying Tai.
\newblock Textcrafter: Accurately rendering multiple texts in complex visual scenes.
\newblock \emph{arXiv preprint arXiv:2503.23461}, 2025.

\bibitem[Esser et~al.(2024)Esser, Kulal, Blattmann, Entezari, M{\"u}ller, Saini, Levi, Lorenz, Sauer, Boesel, et~al.]{esser2024scaling}
Patrick Esser, Sumith Kulal, Andreas Blattmann, Rahim Entezari, Jonas M{\"u}ller, Harry Saini, Yam Levi, Dominik Lorenz, Axel Sauer, Frederic Boesel, et~al.
\newblock Scaling rectified flow transformers for high-resolution image synthesis.
\newblock In \emph{Forty-first international conference on machine learning}, 2024.

\bibitem[Fan et~al.(2025)Fan, Nan, Xie, Zhou, Yang, Fu, Li, Yang, and Tai]{fan2025instancecap}
Tiehan Fan, Kepan Nan, Rui Xie, Penghao Zhou, Zhenheng Yang, Chaoyou Fu, Xiang Li, Jian Yang, and Ying Tai.
\newblock Instancecap: Improving text-to-video generation via instance-aware structured caption.
\newblock In \emph{Proceedings of the Computer Vision and Pattern Recognition Conference}, pages 28974--28983, 2025.

\bibitem[Goodfellow et~al.(2020)Goodfellow, Pouget-Abadie, Mirza, Xu, Warde-Farley, Ozair, Courville, and Bengio]{goodfellow2020generative}
Ian Goodfellow, Jean Pouget-Abadie, Mehdi Mirza, Bing Xu, David Warde-Farley, Sherjil Ozair, Aaron Courville, and Yoshua Bengio.
\newblock Generative adversarial networks.
\newblock \emph{Communications of the ACM}, 63\penalty0 (11):\penalty0 139--144, 2020.

\bibitem[Goodfellow et~al.(2014)Goodfellow, Pouget-Abadie, Mirza, Xu, Warde-Farley, Ozair, Courville, and Bengio]{goodfellow2014generative}
Ian~J Goodfellow, Jean Pouget-Abadie, Mehdi Mirza, Bing Xu, David Warde-Farley, Sherjil Ozair, Aaron Courville, and Yoshua Bengio.
\newblock Generative adversarial nets.
\newblock \emph{Advances in neural information processing systems}, 27, 2014.

\bibitem[Goyal and Bengio(2022)]{goyal2022inductive}
Anirudh Goyal and Yoshua Bengio.
\newblock Inductive biases for deep learning of higher-level cognition.
\newblock \emph{Proceedings of the Royal Society A}, 478\penalty0 (2266):\penalty0 20210068, 2022.

\bibitem[Gupta et~al.(2024)Gupta, Yu, Sohn, Gu, Hahn, Li, Essa, Jiang, and Lezama]{gupta2024photorealistic}
Agrim Gupta, Lijun Yu, Kihyuk Sohn, Xiuye Gu, Meera Hahn, Fei-Fei Li, Irfan Essa, Lu Jiang, and Jos{\'e} Lezama.
\newblock Photorealistic video generation with diffusion models.
\newblock In \emph{European Conference on Computer Vision}, pages 393--411. Springer, 2024.

\bibitem[Heusel et~al.(2017)Heusel, Ramsauer, Unterthiner, Nessler, and Hochreiter]{heusel2017gans}
Martin Heusel, Hubert Ramsauer, Thomas Unterthiner, Bernhard Nessler, and Sepp Hochreiter.
\newblock Gans trained by a two time-scale update rule converge to a local nash equilibrium.
\newblock \emph{Advances in neural information processing systems}, 30, 2017.

\bibitem[Ho and Salimans(2022)]{ho2022classifier}
Jonathan Ho and Tim Salimans.
\newblock Classifier-free diffusion guidance.
\newblock \emph{arXiv preprint arXiv:2207.12598}, 2022.

\bibitem[Ho et~al.(2020)Ho, Jain, and Abbeel]{ho2020denoising}
Jonathan Ho, Ajay Jain, and Pieter Abbeel.
\newblock Denoising diffusion probabilistic models.
\newblock \emph{Advances in neural information processing systems}, 33:\penalty0 6840--6851, 2020.

\bibitem[Ho et~al.(2022)Ho, Saharia, Chan, Fleet, Norouzi, and Salimans]{ho2022cascaded}
Jonathan Ho, Chitwan Saharia, William Chan, David~J Fleet, Mohammad Norouzi, and Tim Salimans.
\newblock Cascaded diffusion models for high fidelity image generation.
\newblock \emph{Journal of Machine Learning Research}, 23\penalty0 (47):\penalty0 1--33, 2022.

\bibitem[Hoogeboom et~al.(2023)Hoogeboom, Heek, and Salimans]{hoogeboom2023simple}
Emiel Hoogeboom, Jonathan Heek, and Tim Salimans.
\newblock simple diffusion: End-to-end diffusion for high resolution images.
\newblock In \emph{International Conference on Machine Learning}, pages 13213--13232. PMLR, 2023.

\bibitem[Hoogeboom et~al.(2024)Hoogeboom, Mensink, Heek, Lamerigts, Gao, and Salimans]{hoogeboom2024simpler}
Emiel Hoogeboom, Thomas Mensink, Jonathan Heek, Kay Lamerigts, Ruiqi Gao, and Tim Salimans.
\newblock Simpler diffusion (sid2): 1.5 fid on imagenet512 with pixel-space diffusion.
\newblock \emph{arXiv preprint arXiv:2410.19324}, 2024.

\bibitem[Jabri et~al.(2022)Jabri, Fleet, and Chen]{jabri2022scalable}
Allan Jabri, David Fleet, and Ting Chen.
\newblock Scalable adaptive computation for iterative generation.
\newblock \emph{arXiv preprint arXiv:2212.11972}, 2022.

\bibitem[Kadkhodaie et~al.(2023)Kadkhodaie, Guth, Simoncelli, and Mallat]{kadkhodaie2023generalization}
Zahra Kadkhodaie, Florentin Guth, Eero~P Simoncelli, and St{\'e}phane Mallat.
\newblock Generalization in diffusion models arises from geometry-adaptive harmonic representations.
\newblock \emph{arXiv preprint arXiv:2310.02557}, 2023.

\bibitem[Karras et~al.(2019)Karras, Laine, and Aila]{karras2019style}
Tero Karras, Samuli Laine, and Timo Aila.
\newblock A style-based generator architecture for generative adversarial networks.
\newblock In \emph{Proceedings of the IEEE/CVF conference on computer vision and pattern recognition}, pages 4401--4410, 2019.

\bibitem[Karras et~al.(2024)Karras, Aittala, Kynk{\"a}{\"a}nniemi, Lehtinen, Aila, and Laine]{karras2024guiding}
Tero Karras, Miika Aittala, Tuomas Kynk{\"a}{\"a}nniemi, Jaakko Lehtinen, Timo Aila, and Samuli Laine.
\newblock Guiding a diffusion model with a bad version of itself.
\newblock \emph{Advances in Neural Information Processing Systems}, 37:\penalty0 52996--53021, 2024.

\bibitem[Kilian et~al.(2024)Kilian, Jampani, and Zettlemoyer]{kilian2024computational}
Maciej Kilian, Varun Jampani, and Luke Zettlemoyer.
\newblock Computational tradeoffs in image synthesis: Diffusion, masked-token, and next-token prediction.
\newblock \emph{arXiv preprint arXiv:2405.13218}, 2024.

\bibitem[Kingma and Gao(2023)]{kingma2023understanding}
Diederik Kingma and Ruiqi Gao.
\newblock Understanding diffusion objectives as the elbo with simple data augmentation.
\newblock \emph{Advances in Neural Information Processing Systems}, 36:\penalty0 65484--65516, 2023.

\bibitem[Kingma and Welling(2013)]{kingma2013auto}
Diederik~P Kingma and Max Welling.
\newblock Auto-encoding variational bayes.
\newblock \emph{arXiv preprint arXiv:1312.6114}, 2013.

\bibitem[Kynk{\"a}{\"a}nniemi et~al.(2019)Kynk{\"a}{\"a}nniemi, Karras, Laine, Lehtinen, and Aila]{kynkaanniemi2019improved}
Tuomas Kynk{\"a}{\"a}nniemi, Tero Karras, Samuli Laine, Jaakko Lehtinen, and Timo Aila.
\newblock Improved precision and recall metric for assessing generative models.
\newblock \emph{Advances in neural information processing systems}, 32, 2019.

\bibitem[Kynk{\"a}{\"a}nniemi et~al.(2024)Kynk{\"a}{\"a}nniemi, Aittala, Karras, Laine, Aila, and Lehtinen]{kynkaanniemi2024applying}
Tuomas Kynk{\"a}{\"a}nniemi, Miika Aittala, Tero Karras, Samuli Laine, Timo Aila, and Jaakko Lehtinen.
\newblock Applying guidance in a limited interval improves sample and distribution quality in diffusion models.
\newblock \emph{Advances in Neural Information Processing Systems}, 37:\penalty0 122458--122483, 2024.

\bibitem[Li and He(2025)]{li2025jit}
Tianhong Li and Kaiming He.
\newblock Back to basics: Let denoising generative models denoise.
\newblock \emph{arXiv preprint arXiv:2511.13720}, 2025.

\bibitem[Ma et~al.(2024)Ma, Goldstein, Albergo, Boffi, Vanden-Eijnden, and Xie]{ma2024sit}
Nanye Ma, Mark Goldstein, Michael~S Albergo, Nicholas~M Boffi, Eric Vanden-Eijnden, and Saining Xie.
\newblock Sit: Exploring flow and diffusion-based generative models with scalable interpolant transformers.
\newblock In \emph{European Conference on Computer Vision}, pages 23--40. Springer, 2024.

\bibitem[Mirza and Osindero(2014)]{mirza2014conditional}
Mehdi Mirza and Simon Osindero.
\newblock Conditional generative adversarial nets.
\newblock \emph{arXiv preprint arXiv:1411.1784}, 2014.

\bibitem[Nan et~al.(2024)Nan, Xie, Zhou, Fan, Yang, Chen, Li, Yang, and Tai]{nan2024openvid}
Kepan Nan, Rui Xie, Penghao Zhou, Tiehan Fan, Zhenheng Yang, Zhijie Chen, Xiang Li, Jian Yang, and Ying Tai.
\newblock Openvid-1m: A large-scale high-quality dataset for text-to-video generation.
\newblock \emph{arXiv preprint arXiv:2407.02371}, 2024.

\bibitem[Nash et~al.(2021)Nash, Menick, Dieleman, and Battaglia]{nash2021generating}
Charlie Nash, Jacob Menick, Sander Dieleman, and Peter~W Battaglia.
\newblock Generating images with sparse representations.
\newblock \emph{arXiv preprint arXiv:2103.03841}, 2021.

\bibitem[Peebles and Xie(2023)]{peebles2023scalable}
William Peebles and Saining Xie.
\newblock Scalable diffusion models with transformers.
\newblock In \emph{Proceedings of the IEEE/CVF International Conference on Computer Vision}, pages 4195--4205, 2023.

\bibitem[Podell et~al.(2023)Podell, English, Lacey, Blattmann, Dockhorn, M{\"u}ller, Penna, and Rombach]{podell2023sdxl}
Dustin Podell, Zion English, Kyle Lacey, Andreas Blattmann, Tim Dockhorn, Jonas M{\"u}ller, Joe Penna, and Robin Rombach.
\newblock Sdxl: Improving latent diffusion models for high-resolution image synthesis.
\newblock \emph{arXiv preprint arXiv:2307.01952}, 2023.

\bibitem[Radford et~al.(2015)Radford, Metz, and Chintala]{radford2015unsupervised}
Alec Radford, Luke Metz, and Soumith Chintala.
\newblock Unsupervised representation learning with deep convolutional generative adversarial networks.
\newblock \emph{arXiv preprint arXiv:1511.06434}, 2015.

\bibitem[Ramesh et~al.(2022)Ramesh, Dhariwal, Nichol, Chu, and Chen]{ramesh2022hierarchical}
Aditya Ramesh, Prafulla Dhariwal, Alex Nichol, Casey Chu, and Mark Chen.
\newblock Hierarchical text-conditional image generation with clip latents.
\newblock \emph{arXiv preprint arXiv:2204.06125}, 1\penalty0 (2):\penalty0 3, 2022.

\bibitem[Rombach et~al.(2022)Rombach, Blattmann, Lorenz, Esser, and Ommer]{rombach2022high}
Robin Rombach, Andreas Blattmann, Dominik Lorenz, Patrick Esser, and Bj{\"o}rn Ommer.
\newblock High-resolution image synthesis with latent diffusion models.
\newblock In \emph{Proceedings of the IEEE/CVF conference on computer vision and pattern recognition}, pages 10684--10695, 2022.

\bibitem[Ronneberger et~al.(2015)Ronneberger, Fischer, and Brox]{ronneberger2015u}
Olaf Ronneberger, Philipp Fischer, and Thomas Brox.
\newblock U-net: Convolutional networks for biomedical image segmentation.
\newblock In \emph{International Conference on Medical image computing and computer-assisted intervention}, pages 234--241. Springer, 2015.

\bibitem[Saharia et~al.(2022)Saharia, Chan, Saxena, Li, Whang, Denton, Ghasemipour, Gontijo~Lopes, Karagol~Ayan, Salimans, et~al.]{saharia2022photorealistic}
Chitwan Saharia, William Chan, Saurabh Saxena, Lala Li, Jay Whang, Emily~L Denton, Kamyar Ghasemipour, Raphael Gontijo~Lopes, Burcu Karagol~Ayan, Tim Salimans, et~al.
\newblock Photorealistic text-to-image diffusion models with deep language understanding.
\newblock \emph{Advances in neural information processing systems}, 35:\penalty0 36479--36494, 2022.

\bibitem[Salimans et~al.(2016)Salimans, Goodfellow, Zaremba, Cheung, Radford, and Chen]{salimans2016improved}
Tim Salimans, Ian Goodfellow, Wojciech Zaremba, Vicki Cheung, Alec Radford, and Xi Chen.
\newblock Improved techniques for training gans.
\newblock \emph{Advances in neural information processing systems}, 29, 2016.

\bibitem[Skorokhodov et~al.(2025)Skorokhodov, Girish, Hu, Menapace, Li, Abdal, Tulyakov, and Siarohin]{skorokhodov2025improving}
Ivan Skorokhodov, Sharath Girish, Benran Hu, Willi Menapace, Yanyu Li, Rameen Abdal, Sergey Tulyakov, and Aliaksandr Siarohin.
\newblock Improving the diffusability of autoencoders.
\newblock \emph{arXiv preprint arXiv:2502.14831}, 2025.

\bibitem[Sohl-Dickstein et~al.(2015)Sohl-Dickstein, Weiss, Maheswaranathan, and Ganguli]{sohl2015deep}
Jascha Sohl-Dickstein, Eric Weiss, Niru Maheswaranathan, and Surya Ganguli.
\newblock Deep unsupervised learning using nonequilibrium thermodynamics.
\newblock In \emph{International conference on machine learning}, pages 2256--2265. PMLR, 2015.

\bibitem[Song et~al.(2020{\natexlab{a}})Song, Meng, and Ermon]{song2020denoising}
Jiaming Song, Chenlin Meng, and Stefano Ermon.
\newblock Denoising diffusion implicit models.
\newblock \emph{arXiv preprint arXiv:2010.02502}, 2020{\natexlab{a}}.

\bibitem[Song et~al.(2020{\natexlab{b}})Song, Sohl-Dickstein, Kingma, Kumar, Ermon, and Poole]{song2020score}
Yang Song, Jascha Sohl-Dickstein, Diederik~P Kingma, Abhishek Kumar, Stefano Ermon, and Ben Poole.
\newblock Score-based generative modeling through stochastic differential equations.
\newblock \emph{arXiv preprint arXiv:2011.13456}, 2020{\natexlab{b}}.

\bibitem[Teng et~al.(2023)Teng, Zheng, Ding, Hong, Wangni, Yang, and Tang]{teng2023relay}
Jiayan Teng, Wendi Zheng, Ming Ding, Wenyi Hong, Jianqiao Wangni, Zhuoyi Yang, and Jie Tang.
\newblock Relay diffusion: Unifying diffusion process across resolutions for image synthesis.
\newblock \emph{arXiv preprint arXiv:2309.03350}, 2023.

\bibitem[Tschannen et~al.(2024)Tschannen, Pinto, and Kolesnikov]{tschannen2024jetformer}
Michael Tschannen, Andr{\'e}~Susano Pinto, and Alexander Kolesnikov.
\newblock Jetformer: An autoregressive generative model of raw images and text.
\newblock \emph{arXiv preprint arXiv:2411.19722}, 2024.

\bibitem[Vaswani et~al.(2017)Vaswani, Shazeer, Parmar, Uszkoreit, Jones, Gomez, Kaiser, and Polosukhin]{vaswani2017attention}
Ashish Vaswani, Noam Shazeer, Niki Parmar, Jakob Uszkoreit, Llion Jones, Aidan~N Gomez, {\L}ukasz Kaiser, and Illia Polosukhin.
\newblock Attention is all you need.
\newblock \emph{Advances in neural information processing systems}, 30, 2017.

\bibitem[Wang et~al.(2024{\natexlab{a}})Wang, Bai, Wang, Qin, Chen, Li, Tang, and Hu]{wang2024instantid}
Qixun Wang, Xu Bai, Haofan Wang, Zekui Qin, Anthony Chen, Huaxia Li, Xu Tang, and Yao Hu.
\newblock Instantid: Zero-shot identity-preserving generation in seconds.
\newblock \emph{arXiv preprint arXiv:2401.07519}, 2024{\natexlab{a}}.

\bibitem[Wang et~al.(2024{\natexlab{b}})Wang, Li, Song, Li, Ge, Zheng, and Wang]{wang2024flowdcn}
Shuai Wang, Zexian Li, Tianhui Song, Xubin Li, Tiezheng Ge, Bo Zheng, and Limin Wang.
\newblock Flowdcn: Exploring dcn-like architectures for fast image generation with arbitrary resolution.
\newblock In \emph{Proceedings of the 38th International Conference on Neural Information Processing Systems}, pages 87959--87977, 2024{\natexlab{b}}.

\bibitem[Wang et~al.(2025{\natexlab{a}})Wang, Gao, Zhu, Huang, and Wang]{wang2025pixnerd}
Shuai Wang, Ziteng Gao, Chenhui Zhu, Weilin Huang, and Limin Wang.
\newblock Pixnerd: Pixel neural field diffusion.
\newblock \emph{arXiv preprint arXiv:2507.23268}, 2025{\natexlab{a}}.

\bibitem[Wang et~al.(2025{\natexlab{b}})Wang, Tian, Huang, and Wang]{wang2025ddt}
Shuai Wang, Zhi Tian, Weilin Huang, and Limin Wang.
\newblock Ddt: Decoupled diffusion transformer.
\newblock \emph{arXiv preprint arXiv:2504.05741}, 2025{\natexlab{b}}.

\bibitem[Xie et~al.(2024)Xie, Chen, Chen, Cai, Lin, Zhang, Li, Lu, and Han]{xie2024sana}
Enze Xie, Junsong Chen, Junyu Chen, Han Cai, Yujun Lin, Zhekai Zhang, Muyang Li, Yao Lu, and Song Han.
\newblock Sana: Efficient high-resolution image synthesis with linear diffusion transformers.
\newblock \emph{arXiv preprint arXiv:2410.10629}, 2024.

\bibitem[Yang et~al.(2024)Yang, Lan, Lu, and Zheng]{yang2024diffusion}
Tao Yang, Cuiling Lan, Yan Lu, and Nanning Zheng.
\newblock Diffusion model with cross attention as an inductive bias for disentanglement.
\newblock \emph{Advances in Neural Information Processing Systems}, 37:\penalty0 82465--82492, 2024.

\bibitem[Yao et~al.(2025)Yao, Yang, and Wang]{yao2025reconstruction}
Jingfeng Yao, Bin Yang, and Xinggang Wang.
\newblock Reconstruction vs. generation: Taming optimization dilemma in latent diffusion models.
\newblock In \emph{Proceedings of the Computer Vision and Pattern Recognition Conference}, pages 15703--15712, 2025.

\bibitem[Ye et~al.(2023)Ye, Zhang, Liu, Han, and Yang]{ye2023ip}
Hu Ye, Jun Zhang, Sibo Liu, Xiao Han, and Wei Yang.
\newblock Ip-adapter: Text compatible image prompt adapter for text-to-image diffusion models.
\newblock \emph{arXiv preprint arXiv:2308.06721}, 2023.

\bibitem[Yu et~al.(2022)Yu, Xu, Koh, Luong, Baid, Wang, Vasudevan, Ku, Yang, Ayan, et~al.]{yu2022scaling}
Jiahui Yu, Yuanzhong Xu, Jing~Yu Koh, Thang Luong, Gunjan Baid, Zirui Wang, Vijay Vasudevan, Alexander Ku, Yinfei Yang, Burcu~Karagol Ayan, et~al.
\newblock Scaling autoregressive models for content-rich text-to-image generation.
\newblock \emph{arXiv preprint arXiv:2206.10789}, 2\penalty0 (3):\penalty0 5, 2022.

\bibitem[Zhang et~al.(2023)Zhang, Rao, and Agrawala]{zhang2023adding}
Lvmin Zhang, Anyi Rao, and Maneesh Agrawala.
\newblock Adding conditional control to text-to-image diffusion models.
\newblock In \emph{Proceedings of the IEEE/CVF International Conference on Computer Vision}, pages 3836--3847, 2023.

\bibitem[Zhang et~al.(2018)Zhang, Isola, Efros, Shechtman, and Wang]{zhang2018unreasonable}
Richard Zhang, Phillip Isola, Alexei~A Efros, Eli Shechtman, and Oliver Wang.
\newblock The unreasonable effectiveness of deep features as a perceptual metric.
\newblock In \emph{Proceedings of the IEEE conference on computer vision and pattern recognition}, pages 586--595, 2018.

\bibitem[Zhang et~al.(2021)Zhang, Sun, Luo, Ji, Zhou, Wu, Huang, and Ji]{zhang2021rstnet}
Xuying Zhang, Xiaoshuai Sun, Yunpeng Luo, Jiayi Ji, Yiyi Zhou, Yongjian Wu, Feiyue Huang, and Rongrong Ji.
\newblock Rstnet: Captioning with adaptive attention on visual and non-visual words.
\newblock In \emph{Proceedings of the IEEE/CVF conference on computer vision and pattern recognition}, pages 15465--15474, 2021.

\bibitem[Zhang et~al.(2024{\natexlab{a}})Zhang, Liu, Li, Zhang, Liu, Wang, Ouyang, Xiong, Gao, Hou, et~al.]{zhang2024tar3d}
Xuying Zhang, Yutong Liu, Yangguang Li, Renrui Zhang, Yufei Liu, Kai Wang, Wanli Ouyang, Zhiwei Xiong, Peng Gao, Qibin Hou, et~al.
\newblock Tar3d: Creating high-quality 3d assets via next-part prediction.
\newblock \emph{arXiv preprint arXiv:2412.16919}, 2024{\natexlab{a}}.

\bibitem[Zhang et~al.(2024{\natexlab{b}})Zhang, Yin, Chen, Lin, Li, Hou, and Cheng]{zhang2024temo}
Xuying Zhang, Bo-Wen Yin, Yuming Chen, Zheng Lin, Yunheng Li, Qibin Hou, and Ming-Ming Cheng.
\newblock Temo: Towards text-driven 3d stylization for multi-object meshes.
\newblock In \emph{Proceedings of the ieee/cvf conference on computer vision and pattern recognition}, pages 19531--19540, 2024{\natexlab{b}}.

\bibitem[Zhang et~al.(2025)Zhang, Zhou, Wang, Wang, Li, Jiao, Zhou, Hou, and Cheng]{zhang2025ar}
Xuying Zhang, Yupeng Zhou, Kai Wang, Yikai Wang, Zhen Li, Shaohui Jiao, Daquan Zhou, Qibin Hou, and Ming-Ming Cheng.
\newblock Ar-1-to-3: Single image to consistent 3d object generation via next-view prediction.
\newblock \emph{arXiv preprint arXiv:2503.12929}, 2025.

\bibitem[Zhao et~al.(2024{\natexlab{a}})Zhao, Cai, Dong, and Hu]{zhao2024wavelet}
Chen Zhao, Weiling Cai, Chenyu Dong, and Chengwei Hu.
\newblock Wavelet-based fourier information interaction with frequency diffusion adjustment for underwater image restoration.
\newblock In \emph{Proceedings of the IEEE/CVF Conference on Computer Vision and Pattern Recognition}, pages 8281--8291, 2024{\natexlab{a}}.

\bibitem[Zhao et~al.(2024{\natexlab{b}})Zhao, Cai, Hu, and Yuan]{zhao2024cycle}
Chen Zhao, Weiling Cai, Chengwei Hu, and Zheng Yuan.
\newblock Cycle contrastive adversarial learning with structural consistency for unsupervised high-quality image deraining transformer.
\newblock \emph{Neural Networks}, 178:\penalty0 106428, 2024{\natexlab{b}}.

\bibitem[Zhao et~al.(2025)Zhao, Ci, Xu, Fan, Guan, Ge, Yang, and Tai]{zhao2025ultrahr}
Chen Zhao, En Ci, Yunzhe Xu, Tiehan Fan, Shanyan Guan, Yanhao Ge, Jian Yang, and Ying Tai.
\newblock Ultrahr-100k: Enhancing uhr image synthesis with a large-scale high-quality dataset.
\newblock \emph{Advances in Neural Information Processing Systems}, 2025.

\bibitem[Zhao et~al.(2026)Zhao, Chen, Li, Kang, Lu, Wei, Zhang, Yang, and Tai]{zhao2026luve}
Chen Zhao, Jiawei Chen, Hongyu Li, Zhuoliang Kang, Shilin Lu, Xiaoming Wei, Kai Zhang, Jian Yang, and Ying Tai.
\newblock Luve: Latent-cascaded ultra-high-resolution video generation with dual frequency experts.
\newblock \emph{arXiv preprint arXiv:2602.11564}, 2026.

\bibitem[Zheng et~al.(2025)Zheng, Zhao, Yang, Xiao, Lin, Wu, Deng, Zhang, and Zhu]{zheng2025farmer}
Guangting Zheng, Qinyu Zhao, Tao Yang, Fei Xiao, Zhijie Lin, Jie Wu, Jiajun Deng, Yanyong Zhang, and Rui Zhu.
\newblock Farmer: Flow autoregressive transformer over pixels.
\newblock \emph{arXiv preprint arXiv:2510.23588}, 2025.

\bibitem[Zheng et~al.(2023)Zheng, Nie, Vahdat, and Anandkumar]{zheng2023fast}
Hongkai Zheng, Weili Nie, Arash Vahdat, and Anima Anandkumar.
\newblock Fast training of diffusion models with masked transformers.
\newblock \emph{arXiv preprint arXiv:2306.09305}, 2023.

\bibitem[Zhou et~al.(2024{\natexlab{a}})Zhou, Li, Ma, Zhang, and Yang]{zhou2024migc}
Dewei Zhou, You Li, Fan Ma, Xiaoting Zhang, and Yi Yang.
\newblock Migc: Multi-instance generation controller for text-to-image synthesis.
\newblock In \emph{Proceedings of the IEEE/CVF Conference on Computer Vision and Pattern Recognition}, pages 6818--6828, 2024{\natexlab{a}}.

\bibitem[Zhou et~al.(2024{\natexlab{b}})Zhou, Xie, Yang, and Yang]{zhou20243dis}
Dewei Zhou, Ji Xie, Zongxin Yang, and Yi Yang.
\newblock 3dis: Depth-driven decoupled instance synthesis for text-to-image generation.
\newblock \emph{arXiv preprint arXiv:2410.12669}, 2024{\natexlab{b}}.

\bibitem[Zhu et~al.(2017)Zhu, Park, Isola, and Efros]{zhu2017unpaired}
Jun-Yan Zhu, Taesung Park, Phillip Isola, and Alexei~A Efros.
\newblock Unpaired image-to-image translation using cycle-consistent adversarial networks.
\newblock In \emph{Proceedings of the IEEE international conference on computer vision}, pages 2223--2232, 2017.

\end{thebibliography}
}

\clearpage
\setcounter{page}{1}

\maketitlesupplementary

\newcommand{\bd}{{\mathbf{d}}}
\newcommand{\x}{{\mathbf{x}}}
\newcommand{\y}{{\mathbf{y}}}
\newcommand{\z}{{\mathbf{z}}}
\newcommand{\w}{{\mathbf{w}}}
\newcommand{\p}{{\mathbf{p}}}
\newcommand{\ba}{{\mathbf{a}}}
\newcommand{\bv}{{\mathbf{v}}}
\newcommand{\bu}{{\mathbf{u}}}
\newcommand{\g}{{\mathbf{g}}}
\newcommand{\h}{{\mathbf{h}}}
\newcommand{\e}{{\mathbf{e}}}
\newcommand{\beps}{{\boldsymbol{\epsilon}}}
\newcommand{\bepss}{{\boldsymbol{\epsilon}}^2}
\newcommand{\bphi}{{\boldsymbol{\phi}}}
\newcommand{\tmu}{{\tilde\mu}}
\newcommand{\mulambda}[1]{\mu_{\lambda, #1}}
\newcommand{\alambda}{a_\lambda}
\newcommand{\tepsilon}{\tilde\epsilon}

\newcommand{\gauss}{{\boldsymbol{\xi}}}
\newcommand{\gz}{{\boldsymbol{\zeta}}}
\newcommand{\argmin}{\mathop{\mathrm{argmin}}}
\newcommand{\bdelta}{{\boldsymbol{\delta}}}

\newcommand{\I}{{\mathbf{I}}}
\newcommand{\A}{{\mathbf{A}}}
\newcommand{\B}{{\mathbf{B}}}
\newcommand{\U}{{\mathbf{U}}}
\newcommand{\D}{{\mathbf{D}}}
\newcommand{\C}{{\mathbf{C}}}
\newcommand{\W}{{\mathbf{W}}}
\newcommand{\V}{{\mathbf{V}}}
\newcommand{\M}{{\mathbf{M}}}
\newcommand{\bXi}{\boldsymbol\Xi}

\newcommand{\bbeta}{\bm{\beta}}

\newcommand{\N}{{\mathbb{N}}}
\newcommand{\R}{{\mathbb{R}}}
\newcommand{\E}{{\mathbb{E}}}
\newcommand{\tlam}{\tilde{\lambda}}

\newcommand{\tB}{{\tilde{\mathbf{B}}}}

\newcommand{\tr}{\mathrm{tr}}

\newcommand{\AGD}{\mathrm{AGD}}
\newcommand{\cO}{\mathcal{O}}
\newcommand{\hntf}{\hat\nabla_\rho\tilde f_\delta}

\newcommand{\effdim}{r}
\newcommand{\efftrace}{\mathrm{ET}}

\newtheorem{theorem}{Theorem}[section]
\newtheorem{proposition}[theorem]{Proposition}
\newtheorem{lemma}[theorem]{Lemma}
\newtheorem{corollary}[theorem]{Corollary}
\newtheorem{definition}[theorem]{Definition}
\newtheorem{assumption}[theorem]{Assumption}
\newtheorem{remark}[theorem]{Remark}
\newcommand{\upA}{\mathbf{A}}
\newcommand{\upB}{\mathbf{B}}
\newcommand{\upC}{\mathbf{C}}
\newcommand{\upD}{\mathbf{D}}
\newcommand{\upF}{\mathbf{F}}
\newcommand{\upG}{\mathbf{G}}
\newcommand{\upH}{\mathbf{H}}
\newcommand{\upI}{\mathbf{I}}
\newcommand{\upJ}{\mathbf{J}}
\newcommand{\upK}{\mathbf{K}}
\newcommand{\upL}{\mathbf{L}}
\newcommand{\upM}{\mathbf{M}}
\newcommand{\upN}{\mathbf{N}}
\newcommand{\upP}{\mathbf{P}}
\newcommand{\upQ}{\mathbf{Q}}
\newcommand{\upV}{\mathbf{V}}
\newcommand{\upX}{\mathbf{X}}
\newcommand{\upY}{\mathbf{Y}}
\newcommand{\upU}{\mathbf{U}}
\newcommand{\upE}{\mathbf{E}}
\appendix

\section{Why Patch Detailer Head: A Theoretical Perspective} \label{app:thm}

In this section, we try to provide a simplified theoretical analysis to further elucidate why we need local detail refinement in enhancing generation quality. From a general insight, we argue that DiT primarily focuses on the layout and arrangement of the dominant elements in the image, or in other words, the low-frequency signals of the global data. Consequently, it is less effective to learn local details and high-frequency signals. Through the refinement structure, we directly inject all signals from the global data into the learning process, which substantially improves the fine-grained processing of these high-frequency details.

Specifically, we follow the flow matching description of the diffusion process. Given an initial data sample $\x_0 \sim p_{{\rm data}}(\x_0) \in \mathbb{R}^d$ as the input, a Gaussian noise $\epsilon \sim \mathcal{N}(0,\mathbf{I}_d)$, and $t \in [0,1]$, let
\begin{equation}
    \x_t = (1-t)\x_0 + t\epsilon.
\end{equation}
In this paper, since all inputs are partitioned into patches of equal size, we first define the patch-level input as follows.

\begin{definition}[Patch-level Input]
    For each input $\x_0 \in \mathbb{R}^d$, we define the patch-level input as $\left\{\x_0^{(s)}\right\}_{s=1}^N$, where $\x_0^{(s)} \in \mathbb{R}^p$ and $\x_0 = \left[\left(\x_0^{(1)}\right)^\top,
    \cdots,\left(\x_0^{(N)}\right)^\top\right]^\top$, $Np=d$. 
\end{definition}

It is natural to represent the patch-level input by a series of selection matrices $\left\{ \upP^{(s)} \right\}_{s=1}^N$. For each $s$, $\upP^{(s)} \in \mathbb{R}^{p\times d}$ satisfies $\upP^{(s)}\left(\upP^{(s)}\right)^\top = \mathbf{I}_p$ and $\upP^{(s)}\x_0 = \x_0^{(s)} $.
The flow-based models try to minimize a loss function defined as $\mathcal{L}_{{\rm FM}} = \mathbb{E}_{t,\x_0, \epsilon} \left[ \left\Vert f(\x_t, t) - (\epsilon -\x_0)\right\Vert^2\right]$. Assuming that each patch is independent of one another, a patch-level predictor $f(\cdot, t)$ tries to estimate the patch-level objective field $\hat v^{(s)} = f\left(\x_t^{(s)},t\right)$ for each patch-level noised input $\x_t^{(s)} = (1-t) \x_{0}^{(s)}+t \epsilon^{(s)}$, where $\epsilon^{(s)} = \upP^{(s)}\epsilon \sim \mathcal{N}(0,\mathbf{I}_p)$. For the given $\mathcal{L}_{{\rm FM}}$, the optimal predictor is the conditional expectation $$\hat v^{(s),*} = \mathbb{E}\left[\epsilon^{(s)} -\x_0^{(s)}\ \middle|\ \x_t^{(s)}  \right].$$

However, in true generation tasks, each patch is not independent of others, because, for natural images, the boundaries between adjacent patches are typically continuous and smoothly varying (e.g., there is little difference between one patch of sky and another). The correlation between patches only weakens when an abrupt transition occurs in the image's elements, such as at the boundary between sky and grass. Moreover, DiT's attention-based structure allows a single patch to access partial information from all other patches. Although this information may be coarse, this remains a complex, coupled structure. Therefore, for a DiT model, the estimate of $\hat v^{(s)}$ is not only based on $\x_t^{(s)}$ but also some other information from $\left\{\x_t^{(l)} \right\}_{l \not= s}$. Thus, we define the effective information below.

\begin{definition}[Effective Information]
    For a patch-level noised input $\left\{\x_t^{(s)}\right\}_{s=1}^N$, we define ${\rm EI}^{(s)} \left(f;\left\{\x_t^{(s)}\right\}_{s=1}^N\right)$ to represent the effective information used for a generation model $f$ to estimate the patch-level vector field $\hat v^{(s)}$ for any $s \in [N]$.
\end{definition}

Assuming that each patch is independent of one another, the patch-level estimate $\hat v^{(s)} = f\left(\x_t^{(s)}, t\right)$ only uses $\x_t^{(s)}$ for prediction, which means ${\rm EI}^{(s)} \left(f;\left\{\x_t^{(s)}\right\}_{s=1}^N\right) = \left\{\x_t^{(s)}\right\}$. Thus the optimal predictor can be more generally formulated as $\hat v^{(s),*} = \mathbb{E}\left[\epsilon^{(s)} -\x_0^{(s)} \ \middle|\ {\rm EI}^{(s)} \left(f;\left\{\x_t^{(s)}\right\}_{s=1}^N\right)  \right]$. For attention-based generation models, we cannot accurately obtain the effective information due to the complex coupling structure. However, based on some standard assumptions on the initial data distribution and some empirical observations, we can still give a brief formulation for the effective information.

\begin{assumption}[Data Distribution]
    For the initial data distribution, we assume that $p_{\rm data} \sim \mathcal{N}(\mu,\mathbf{\Sigma})$, where $\mathbf{\Sigma} = \upU \mathbf{\Lambda} \upU^\top$, $\upU = [\bu_1, \cdots, \bu_d]$, and $\mathbf{\Lambda} = {\rm diag}\{\lambda_1,\cdots,\lambda_d\}$.
    \label{ass:pdata}
\end{assumption}
\begin{assumption}[Eigenvalue Decay] There exists $\alpha >1 $ such that for any $i \in [d] $, the eigenvalues of $\mathbf{\Sigma}$ satisfies $\lambda_i \asymp i^{-a} $. 
\label{ass:decay}
\end{assumption}

Assumption \ref{ass:pdata} and \ref{ass:decay} characterize the data distribution as a Gaussian distribution with a covariance of a series of fast-decay eigenvalues. The eigenvalue decay of covariance characterizes the differences in high- and low-frequency signals of the image information. This is consistent with the empirical observation that DiT can effectively learn low-frequency signals but has difficulty capturing high-frequency signals. Given $b >0$, we can decompose the input $\x_0$ into low- and high-frequency components as 
\begin{equation}
    \x_0 = \mu + \x_{0, {\rm low}} + \x_{0, {\rm high}},
\end{equation}
where $\x_{0, {\rm low}} \sim \mathcal{N}(0, \mathbf{\Sigma}_{{\rm low}})$ and $\x_{0, {\rm high}} \sim \mathcal{N}(0, \mathbf{\Sigma}_{{\rm high}})$. $ \mathbf{\Sigma}_{{\rm low}} $ satisfies $\mathbf{\Sigma}_{{\rm low}} = \U_r \mathbf{\Lambda}_r \U_r^\top = \sum_{i=1}^r \lambda_i \bu_i \bu_i^\top$ where $\lambda_r > b$ and $\lambda_{r+1} \leq b$, and $\mathbf{\Sigma}_{{\rm high}} = \mathbf{\Sigma} - \mathbf{\Sigma}_{{\rm low}}$. Thus we can decompose the patch-level noised input $\x_t^{(s)}$ as
\begin{equation}
        \x_t^{(s)} = \underbrace{(1-t) \upP^{(s)} \mu}_{{\rm Mean}^{(s)}} + \underbrace{(1-t) \upP^{(s)} \x_{0, {\rm low}}}_{{\rm Low}^{(s)}} + \underbrace{(1-t) \upP^{(s)} \x_{0, {\rm high}}}_{{\rm High}^{(s)}} + \underbrace{t \upP^{(s)} \epsilon}_{{\rm Noise}^{(s)}}
    \end{equation}
We can assume that for the DiT model, the effective information is composed of the local patch itself and the low-frequency signals of other patches as below.

\begin{assumption}[EI of DiT]
    Given DiT as the predictor, there exists $\beta >0$ such that for any $s \in [N]$, the effective information to estimate the patch-level vector $\hat v^{(s)}$ satisfies 
    \begin{equation}
        {\rm EI}^{(s)} \left({\rm DiT};\left\{\x_t^{(s)}\right\}_{s=1}^N\right) =  \left\{\x_t^{(s)}\right\} \cup \left\{\x_{t, {\rm low}}^{(l)}\right\}_{l \neq s},
    \end{equation}
    where 
    \begin{equation}
        \x_{t,{\rm low}}^{(l)} = {\rm Mean}^{(l)} + {\rm Low}^{(l)} + {\rm Noise}^{(l)}
    \end{equation} for all $l \not= s$.
    \label{ass:eidit}
\end{assumption}

Our refinement structure directly injects all signals from the initial data $\x_0$ for prediction, which means that for DiP, the effective information satisfies
\begin{equation}
    {\rm EI}^{(s)} \left({\rm DiP};\left\{\x_t^{(s)}\right\}_{s=1}^N\right) = {\rm EI}^{(s)} \left({\rm DiT};\left\{\x_t^{(s)}\right\}_{s=1}^N\right)  \cup \left\{\x_t^{(s)}\right\}_{s=1}^N = \left\{\x_t^{(s)}\right\}_{s=1}^N.
    \label{equ:eidip}
\end{equation}
We define $\hat v^{(s)}_{{\rm DiT}} = \mathbb{E}\left[\epsilon^{(s)} -\x_0^{(s)} \ \middle|\ {\rm EI}^{(s)} \left({\rm DiT};\left\{\x_t^{(s)}\right\}_{s=1}^N\right) \right]$ and $\hat v^{(s)}_{{\rm DiP}} = \mathbb{E}\left[\epsilon^{(s)} -\x_0^{(s)} \ \middle|\ {\rm EI}^{(s)} \left({\rm DiP};\left\{\x_t^{(s)}\right\}_{s=1}^N\right) \right]$ as the general near-optimal estimate of DiT and DiP, respectively. Then we obtain the main results below.

\begin{theorem}
    Assume that Assumption \ref{ass:pdata}, \ref{ass:decay} and \ref{ass:eidit} hold. Consider using DiT and DiP for the diffusion generation task as the predictor, respectively. The general near-optimal estimate $\hat v^{(s)}_{{\rm DiT}}$ and $\hat v^{(s)}_{{\rm DiP}}$ satisfy
    \begin{equation}
        \hat v^{(s)}_{{\rm DiT}} =\upP^{(s)} \hat\upB\hat\upM\left(\x_t - (1-t) \mu\right) -\upP^{(s)}\mu,
        \label{equ:vdit}
    \end{equation}
    and
    \begin{equation}
        \hat v^{(s)}_{{\rm DiP}} =\upP^{(s)} \upA\upM\left(\x_t - (1-t) \mu\right) -\upP^{(s)}\mu,
        \label{equ:vdip}
    \end{equation}
    respectively, where 
    \begin{equation}
        \begin{split}
            &\hat\upM = \left[ (1-t)^2 \mathbf{\Sigma}_{\rm low} + t^2 \mathbf{I}_d + (1-t)^2 \left(\upP^{(s)}\right)^\top\upP^{(s)}\mathbf{\Sigma}_{\rm high}\left(\upP^{(s)}\right)^\top\upP^{(s)} \right]^{-1}, \\
            &\hat\upB = t\mathbf{I}_d -(1-t)\mathbf{B}, \quad \mathbf{B}=\mathbf{\Sigma}_{\rm low} + \mathbf{\Sigma}_{\rm high}\left(\upP^{(s)}\right)^\top\upP^{(s)},\\
            &\upM = \left[ (1-t)^2 \mathbf{\Sigma} + t^2 \mathbf{I}_d \right]^{-1}, \\
            &\mathbf{A} = t\mathbf{I}_d -(1-t)\mathbf{\Sigma}.
        \end{split}
    \end{equation}
 The denoising operator $\upP^{(s)}\hat\upB\hat\upM$ and $\upP^{(s)} \upA\upM$ satisfies
     \begin{equation}
        \upP^{(s)}\hat\upB\hat\upM \asymp  \sum_{i=1}^r \frac{t-(1-t)\lambda_i}{(1-t)^2\lambda_i+t^2}\bv_i\bu_i^\top + \sum_{i=r+1}^d \frac{\lambda_i}{t}\bv_i\bu_i^\top + \mathcal{I}_1 + \mathcal{I}_2,
        \label{equ:ditdenoiser}
    \end{equation}
    and
    \begin{equation}
        \upP^{(s)}\mathbf{A}\upM = \sum_{i=1}^d \frac{t-(1-t)\lambda_i}{(1-t)^2\lambda_i +t^2}\bv_i\bu_i^\top,
        \label{equ:dipdenoiser}
    \end{equation}
    respectively, where $[\bv_1,\cdots, \bv_d] =\upP^{(s)}[\bu_1,\cdots, \bu_d] $, $\mathcal{I}_1 = -\sum_{i=r+1}^d \sum_{j=1}^r \frac{(1-t)\lambda_i}{(1-t)^2\lambda_j +t^2}\left[\bu_i^\top\left(\upP^{(s)}\right)^\top\upP^{(s)}\bu_j\right]\bv_i\bu_j^\top$ and $\mathcal{I}_2 = -\sum_{i=r+1}^d \sum_{j=r+1}^d \frac{(1-t)\lambda_i}{t^2}\left[\bu_i^\top\left(\upP^{(s)}\right)^\top\upP^{(s)}\bu_j\right]\bv_i\bu_j^\top$.
    \label{thm:main}
\end{theorem}
\begin{proof}
    See Appendix \ref{app:proof}.
\end{proof}


\begin{remark}
    Theorem \ref{thm:main} illustrates that our designed refinement mechanism exhibits a strong adaptive correction capability for high-frequency signals within the image. Specifically, \eqref{equ:vdit} and \eqref{equ:vdip} align with our objective to estimate the conditional expectation of $\epsilon -\x_0$ through DiT and DiP. The first term of \eqref{equ:vdit} and \eqref{equ:vdip} represents the estimate of noise, while the second term means the estimate of original data. We focus on the first term regarded as the ``denoising process'', with respective ``denoising operator'' \eqref{equ:ditdenoiser} and \eqref{equ:dipdenoiser} serve as a global representation characterization, performing denoising on different frequency-domain components derived from the original image.
    \begin{itemize}
        \item When using only DiT for estimation, Equation \eqref{equ:ditdenoiser} indicates that DiT can achieve a good adaptive fit for low-frequency signals ($i \leq r$, dominant signals). However, for high-frequency components in the image, relying solely on DiT may not provide sufficient representational capacity for these components. Specifically, 
        \begin{itemize}
            \item The first term in \eqref{equ:ditdenoiser} corresponds to the denoising process applied to the low-frequency signals. Although all these belong to the low-frequency regime, those with larger values of 
$\lambda_i $ ($i \leq r$) contain a stronger proportion of the original image content relative to noise. Consequently, as 
$\lambda_i$ decreases, the denoising operator adaptively selects a larger correction magnitude. (It can be readily demonstrated that 
$ \frac{t-(1-t)\lambda_i}{(1-t)^2\lambda_i+t^2}$ exhibits a monotonic increase as $\lambda_i$ decreases.) 
\item  The remaining three terms correspond to the denoising process applied to the high-frequency components. Among them, $\mathcal{I}_1$ represents the consistent influence exerted by the low-frequency signal on the high-frequency ones. Since the high-frequency components are associated with $\lambda_i$ ($i \ge r+1$) that are significantly smaller than those of the low-frequency regime (Assumption \ref{ass:decay}), this term can generally be regarded as $o(\mathbf{1})$. The second term and $\mathcal{I}_2$—particularly the second term—point to the following fact: when $t \to 1$ in the early stage of denoising, the magnitude of $\lambda_i$ is much smaller than $t$, leading the model to apply only negligible corrections to the high-frequency components. This weakens the model’s ability to learn from this portion of the signal. Conversely, as $t \to 0$ in the late stage of denoising, where the model aims to learn a sensitive compensatory mechanism to capture more fine-grained details from the original image, $t$ becomes much smaller than $\lambda_i$, causing the model’s corrections to high-frequency components to lose stability. As a result, these corrections may introduce inconsistencies with the previously learned representations, potentially affecting fine details of the final output.

        \end{itemize}

        \item In contrast, when using DiP for estimation, Equation \eqref{equ:dipdenoiser} demonstrates that DiP can provide a robust adaptive correction for all signals. This is attributed to our refinement mechanism, which, at a low computational cost, enhances the effective information during the denoising process. Particularly for high-frequency signals, the refinement provides a powerful supplement to the information that DiT struggles to capture, which aligns with both our intuition and experimental results.
    \end{itemize}
\end{remark}

\section{Proof of Theorem \ref{thm:main}}\label{app:proof}
\begin{proof}
Based on \eqref{equ:eidip}, we have $\hat v^{(s)}_{{\rm DiP}} = \mathbb{E}\left[ \epsilon^{(s)} - \x_0^{(s)} \ \middle|\ \left\{\x_t^{(s)}\right\}_{s=1}^N \right] = \mathbb{E}\left[ \epsilon^{(s)} \ \middle|\ \x_t \right] - \mathbb{E}\left[ \x_0^{(s)} \ \middle|\ \x_t \right]$ . We first obtain the following statistics to obtain the first term $\mathbb{E}\left[ \epsilon^{(s)} \ \middle|\ \x_t \right]$.\\
Expectations: 
\begin{equation}
    \begin{split}
        \mathbb{E}\left[\epsilon^{(s)}\right] = 0, \ \ \mathbb{E}\left[\x_t\right] = (1-t) \mu, 
        \end{split}
\end{equation}
Covariances:
\begin{equation}
    \begin{split}
        \text{Cov}\left(\epsilon^{(s)}, \x_t\right) &= \text{Cov}\left(\upP^{(s)} \epsilon, (1-t) \x_0 + t \epsilon\right) = t \upP^{(s)} \text{Cov}(\epsilon, \epsilon) = t\upP^{(s)},
    \end{split}
\end{equation}
and 
\begin{equation}
    \begin{split}
        \text{Cov}(\x_t) &= \text{Cov}\left((1-t) \x_0 + t \epsilon\right) =(1-t)^2 \text{Cov}(\x_0) + t^2\text{Cov}(\epsilon) = (1-t)^2 \mathbf{\Sigma} + t^2 \mathbf{I}_d.
    \end{split}
\end{equation}
Then we use $\mathbb{E}[Y|X] = \mathbb{E}Y + \text{Cov}(Y,X) \text{Cov}(X,X)^{-1} (X - \mathbb{E}X)$ to obtain that 
\begin{equation}
    \mathbb{E}\left[ \epsilon^{(s)} \ \middle|\ \x_t \right] =t \upP^{(s)} \left[ (1-t)^2 \mathbf{\Sigma} + t^2 \mathbf{I}_d \right]^{-1} \left(\x_t - (1-t) \mu\right).
\end{equation}
Similarly, for the second term of $\hat v^{(s)}_{{\rm DiP}}$ we have \begin{equation}
    \begin{split}
        \text{Cov}\left(\x_0^{(s)}, \x_t\right) &= \text{Cov}\left(\upP^{(s)} \x_0, (1-t) \x_0 + t \epsilon\right) = (1-t) \upP^{(s)} \text{Cov}(\x_0, \x_0) = (1-t)\upP^{(s)}\mathbf{\Sigma},
    \end{split}
\end{equation} 
Then we obtain
\begin{equation}
    \mathbb{E}\left[ \x_0^{(s)} \ \middle|\ \x_t \right] = \upP^{(s)}\mu + (1-t)\upP^{(s)}\mathbf{\Sigma}\left[ (1-t)^2 \mathbf{\Sigma} + t^2 \mathbf{I}_d \right]^{-1} \left(\x_t - (1-t) \mu\right).
\end{equation}
Thus we have 
\begin{equation}
\begin{split}
    \hat v^{(s)}_{{\rm DiP}} &= \mathbb{E}\left[ \epsilon^{(s)} \ \middle|\ \x_t \right] - \mathbb{E}\left[ \x_0^{(s)} \ \middle|\ \x_t \right] \\
    &=\upP^{(s)} \left[t\mathbf{I}_d -(1-t)\mathbf{\Sigma}\right]\left[ (1-t)^2 \mathbf{\Sigma} + t^2 \mathbf{I}_d \right]^{-1} \left(\x_t - (1-t) \mu\right) -\upP^{(s)}\mu.
\end{split}
\end{equation}
Letting $\upM = \left[ (1-t)^2 \mathbf{\Sigma} + t^2 \mathbf{I}_d \right]^{-1}$, $[\bv_1,\cdots, \bv_d] =\upP^{(s)}[\bu_1,\cdots, \bu_d] $, $\mathbf{A} = t\mathbf{I}_d -(1-t)\mathbf{\Sigma}$, we have 
\begin{equation}
\begin{split}
    \upP^{(s)}\mathbf{A}\upM &= \left(\sum_{j=1}^d \bv_j\bu_j^\top\right) \left( \sum_{i=1}^d \frac{t-(1-t)\lambda_i}{(1-t)^2\lambda_i +t^2}\bu_i\bu_i^\top\right) \\
    &= \sum_{i=1}^d \frac{t-(1-t)\lambda_i}{(1-t)^2\lambda_i +t^2}\bv_i\bu_i^\top.
\end{split}
\end{equation}

Similarly, based on Assumption \ref{ass:eidit}, we have $\hat v^{(s)}_{{\rm DiT}} = \mathbb{E}\left[ \epsilon^{(s)} -\x_0^{(s)} \ \middle|\ \left\{\x_t^{(s)}\right\} \cup \left\{\x_{t, {\rm low}}^{(l)}\right\}_{l \neq s} \right]$, where $\x_{t,{\rm low}}^{(l)} = (1-t) \upP^{(l)} \mu + (1-t)\upP^{(l)} \x_{0, {\rm low}} + t \upP^{(l)} \epsilon$. We first use one vector to represent the condition $\left\{\x_t^{(s)}\right\} \cup \left\{\x_{t, {\rm low}}^{(l)}\right\}_{l \neq s}$.  We try to construct an observation $\hat\x_t$ such that at $s$ patch, $\upP^{(s)}\hat\x_t = \x_t^{(s)}$, and at $l\not=s$ patch $\upP^{(l)}\hat\x_t = \x_{t,{\rm low}}^{(l)}$. The following $\hat\x_t$ satisfies the requirement above
\begin{equation}
    \hat\x_t = (1-t)\mu + (1-t) \x_{0, {\rm low}} + t \epsilon + (1-t) \left(\upP^{(s)}\right)^\top\upP^{(s)} \x_{0, {\rm high}}.
\end{equation}
Now we use the same technique to obtain $\mathbb{E}\left[ \epsilon^{(s)} \ \middle|\ \hat\x_t\right]$. The covariance terms satisfy
\begin{equation}
    \begin{split}
        \text{Cov}(\epsilon^{(s)}, \hat\x_t) &= \text{Cov}\left(\upP^{(s)} \epsilon,  t \epsilon\right) = t \upP^{(s)},
    \end{split}
\end{equation}
and
\begin{equation}
    \begin{split}
        \text{Cov}(\hat\x_t) &= \text{Cov}\left((1-t)\x_{0, {\rm low}} + t \epsilon + (1-t) \left(\upP^{(s)}\right)^\top\upP^{(s)} \x_{0, {\rm high}}\right) \\
        &=(1-t)^2 \text{Cov}(\x_{0, {\rm low}}) + t^2 \text{Cov}(\epsilon) + (1-t)^2 \left(\upP^{(s)}\right)^\top\upP^{(s)}\text{Cov}(\x_{0, {\rm high}})\left(\upP^{(s)}\right)^\top\upP^{(s)}\\
        &= (1-t)^2 \mathbf{\Sigma}_{\rm low} + t^2 \mathbf{I}_d + (1-t)^2 \left(\upP^{(s)}\right)^\top\upP^{(s)}\mathbf{\Sigma}_{\rm high}\left(\upP^{(s)}\right)^\top\upP^{(s)}.
    \end{split}
\end{equation}
Thus we have
\begin{equation}
    \mathbb{E}\left[ \epsilon^{(s)} \ \middle|\ \hat\x_t\right] =t\upP^{(s)} \left[ (1-t)^2 \mathbf{\Sigma}_{\rm low} + t^2 \mathbf{I}_d + (1-t)^2\left(\upP^{(s)}\right)^\top\upP^{(s)}\mathbf{\Sigma}_{\rm high}\left(\upP^{(s)}\right)^\top\upP^{(s)} \right]^{-1} \left(\hat\x_t - (1-t) \mu\right).
\end{equation}
For $\mathbb{E}\left[ \x_0^{(s)} \ \middle|\ \hat\x_t\right]$, we have
\begin{equation}
    \begin{split}
        \text{Cov}\left(\x_0^{(s)}, \hat\x_t\right) &= \text{Cov}\left(\upP^{(s)} \x_{0, {\rm low}} + \upP^{(s)} \x_{0, {\rm high}}, (1-t)\mu + (1-t) \x_{0, {\rm low}} + t \epsilon + (1-t) \left(\upP^{(s)}\right)^\top\upP^{(s)} \x_{0, {\rm high}}\right)\\
        &=\text{Cov}\left(\upP^{(s)} \x_{0, {\rm low}}, (1-t) \x_{0, {\rm low}}\right) + \text{Cov}\left(\upP^{(s)} \x_{0, {\rm high}}, \left(\upP^{(s)}\right)^\top\upP^{(s)} \x_{0, {\rm high}}\right) \\
        &= (1-t)\upP^{(s)}\left[\mathbf{\Sigma}_{\rm low} + \mathbf{\Sigma}_{\rm high}\left(\upP^{(s)}\right)^\top\upP^{(s)} \right].
    \end{split}
\end{equation} 
Thus we obtain 
\begin{equation}
    \begin{split}
        \mathbb{E}\left[ \x_0^{(s)} \ \middle|\ \hat\x_t\right] = &\upP^{(s)}\mu + (1-t)\upP^{(s)}\left[\mathbf{\Sigma}_{\rm low} + \mathbf{\Sigma}_{\rm high}\left(\upP^{(s)}\right)^\top\upP^{(s)} \right] \times \\
        &\left[ (1-t)^2 \mathbf{\Sigma}_{\rm low} + t^2 \mathbf{I}_d + (1-t)^2\left(\upP^{(s)}\right)^\top\upP^{(s)}\mathbf{\Sigma}_{\rm high}\left(\upP^{(s)}\right)^\top\upP^{(s)} \right]^{-1}\left(\x_t - (1-t) \mu\right).
    \end{split}
\end{equation}
Finally we have 
\begin{equation}
\begin{split}
    \hat v^{(s)}_{{\rm DiT}} &= \mathbb{E}\left[ \epsilon^{(s)} \ \middle|\ \hat\x_t \right] - \mathbb{E}\left[ \x_0^{(s)} \ \middle|\ \hat\x_t \right] \\
    &=\upP^{(s)} \left[t\mathbf{I}_d -(1-t)\mathbf{B}\right]\hat\upM\left(\x_t - (1-t) \mu\right) -\upP^{(s)}\mu,
\end{split}
\end{equation}
where $\mathbf{B}=\mathbf{\Sigma}_{\rm low} + \mathbf{\Sigma}_{\rm high}\left(\upP^{(s)}\right)^\top\upP^{(s)}$ and $\hat\upM = \left[ (1-t)^2 \mathbf{\Sigma}_{\rm low} + t^2 \mathbf{I}_d + (1-t)^2 \left(\upP^{(s)}\right)^\top\upP^{(s)}\mathbf{\Sigma}_{\rm high}\left(\upP^{(s)}\right)^\top\upP^{(s)} \right]^{-1}$.

Letting $\hat\upB = t\mathbf{I}_d -(1-t)\mathbf{B}$, we have
\begin{equation}
\begin{split}
    \upP^{(s)}\hat\upB\hat\upM &= \left(\sum_{j=1}^d \bv_j\bu_j^\top\right)\left(\hat\upC_1+ \hat\upC_2 \right)\left(\hat\upD +\hat\upE\right)^{-1},
\end{split}
\end{equation}
where 
\begin{equation}
\begin{split}
&\hat\upC_1 = \sum_{i=1}^r \left(t-(1-t)\lambda_i\right)\bu_i\bu_i^\top\\
    &\hat\upC_2 =t\sum_{i=r+1}^d \lambda_i\bu_i\bu_i^\top-(1-t)\sum_{i=r+1}^d \lambda_i \bu_i\bu_i^\top\left(\upP^{(s)}\right)^\top\upP^{(s)}, \\
    &\hat\upD =  \sum_{i=1}^r ((1-t)^2 \lambda_i + t^2) \bu_i \bu_i^\top + \sum_{i=r+1}^d t^2 \bu_i \bu_i^\top, \\
    &\hat\upE = (1-t)^2 \sum_{i=r+1}^d \lambda_i \left(\upP^{(s)}\right)^\top\upP^{(s)} \bu_i \bu_i^\top \left(\upP^{(s)}\right)^\top\upP^{(s)}.
\end{split}
\end{equation} We notice that $\hat\upD$ is a positive diagonal matrix and $\hat\upD^{-1}\hat\upE \asymp o(1)$ because Assumption \ref{ass:decay} shows that $\lambda_{p}/\lambda_{q} \asymp (q/p)^a \asymp o(1)$ for any $1 \leq q \leq r$ and $p \geq r+1$. Thus due to first-order Taylor expansion we have
\begin{equation}
    \begin{split}
        \left(\hat\upD +\hat\upE\right)^{-1} = \left( \mathbf{I}_d + \hat\upD^{-1}\hat\upE\right)^{-1}\hat\upD^{-1} 
        \approx \left(\mathbf{I}_d - \hat\upD^{-1}\hat\upE\right)\hat\upD^{-1} 
        =\hat\upD^{-1} - \hat\upD^{-1}\hat\upE\hat\upD^{-1} 
        \asymp \hat\upD^{-1}.
    \end{split}
\end{equation}
Therefore we obtain
\begin{equation}
\begin{split}
    \upP^{(s)}\hat\upB\hat\upM &\asymp  \left(\sum_{j=1}^d \bv_j\bu_j^\top\right)\left(\hat\upC_1+ \hat\upC_2 \right)\hat\upD^{-1} \\
    &=\sum_{i=1}^r \frac{t-(1-t)\lambda_i}{(1-t)^2\lambda_i+t^2}\bv_i\bu_i^\top + \sum_{i=r+1}^d \frac{\lambda_i}{t}\bv_i\bu_i^\top \\
    &~~~~~~\underbrace{-\sum_{i=r+1}^d \sum_{j=1}^r \frac{(1-t)\lambda_i}{(1-t)^2\lambda_j +t^2}\left[\bu_i^\top\left(\upP^{(s)}\right)^\top\upP^{(s)}\bu_j\right]\bv_i\bu_j^\top}_{\mathcal{I}_1} \\
    &~~~~~~\underbrace{-\sum_{i=r+1}^d \sum_{j=r+1}^d \frac{(1-t)\lambda_i}{t^2}\left[\bu_i^\top\left(\upP^{(s)}\right)^\top\upP^{(s)}\bu_j\right]\bv_i\bu_j^\top}_{\mathcal{I}_2}. \\
\end{split}
\end{equation}
We finish the proof.
\end{proof}

\clearpage
\section{More Implementation Details}

\begin{table}[h]
	\centering
	\setlength{\tabcolsep}{8pt} 
	\footnotesize  
	\renewcommand\arraystretch{1.2}  

	\begin{tabular}{lc}
		\toprule

		\textbf{DiT Architecture}           \\
		Input dim                          &256$\times$256$\times$3\\
        Num. layers                        &26 \\
        Hidden dim.                        &1152 \\
        Num. heads                         &16  \\
		\midrule
		\textbf{Patch Detailer Head Architecture}           \\
        DownSampling path               &16$\rightarrow$8$\rightarrow$4$\rightarrow$2$\rightarrow$1 \\
        UpSampling path                 &1$\rightarrow$2$\rightarrow$4$\rightarrow$8$\rightarrow$16 \\
        DownSampling channel            &3$\rightarrow$64$\rightarrow$128$\rightarrow$256$\rightarrow$512 \\
        Bottleneck                      &(512+1152)$\rightarrow$512 \\
        UpSampling channel              &512$\rightarrow$256$\rightarrow$128$\rightarrow$64$\rightarrow$64 \\
        Output Layer                    &64→3 \\
        \midrule
        \textbf{Optimization}  \\
        Optimizer &AdamW  \\
        Learning rate &0.0001 \\
        Weight decay &0 \\
        Batch size &256 \\
        \midrule
        \textbf{Interpolants}  \\
        Diffusion sampler &Euler \\
        Diffusion steps &100 \\
        Evaluation suite &ADM \\
		\bottomrule
	\end{tabular}
	\caption{Hyperparameter settings.}  
    \vspace{-5mm}
	\label{Tab: config}
\end{table}

\paragraph{Hypermarameters.} Table~\ref{Tab: config} reports the detailed hyperparameters of DiP, including the DiT Architecture, Patch Detailer Head Architecture, Optimization, and Interpolants.

\paragraph{Objective.} DiP follows the training objectives of DDT~\cite{wang2025ddt}. It is trained using flow matching as the objective function and regularized using representation alignment techniques. Further improvements could be made by introducing adversarial loss ~\cite{goodfellow2014generative}, perceptual loss~\cite{zhang2018unreasonable}.

\paragraph{Sampler.} We use Euler-Maruyama ODE sampler with 100 sampling steps by default. For DiT-only and DiP, we used the same inference hyperparameters.

\paragraph{Classifier-Free Guidance.} In our experiments, we employ Interval-based Classifier-Free Guidance~\cite{kynkaanniemi2024applying} (Interval-CFG). Specifically, we set the guidance scale to cfg=2.9. The guidance is activated exclusively within the normalized timestep interval of [0.11, 0.97].

\section{How to Preserving High-Frequency Signal: Patch or Image}

While the theoretical analysis in Appendix \ref{app:thm} establishes the need for all-frequency raw signals to refine missing high-frequency details, we also focus on how can this information be injected in the most effective way. Specifically, we are interested in whether patch-level input is better than image-level input, or vise versa, as shown in Figure \ref{Fig: patch_and_image_input}. Intuitively, the transformer structure in DiT has captured the long-distance dependencies, therefore we only need the specific high-frequency signals or details of the image. A toy experiment in Figure \ref{Fig:toy_exp} verifies this intuition.

In Figure~\ref{Fig:toy_exp}(a), through the patch-level input, the learned manifold (black) tightly adheres to the ground truth structure (orange), effectively capturing intricate branching patterns and sharp boundaries. In contrast, Figure~\ref{Fig:toy_exp}(b) reveals that global processing leads to over-smoothing. The learned distribution is more dispersed and struggles to lock onto fine structural details. 
This suggests that we only need the refinement structure to dedicate its capacity to high-frequency sensing without being distracted by long-distance dependencies. On the contrary, with image-level input the network tends to average out features across a broader spatial regime, resulting in a loss of sharp details in high-frequency regions.

\begin{figure*}[h]
    \centering
    \includegraphics[width=\linewidth]{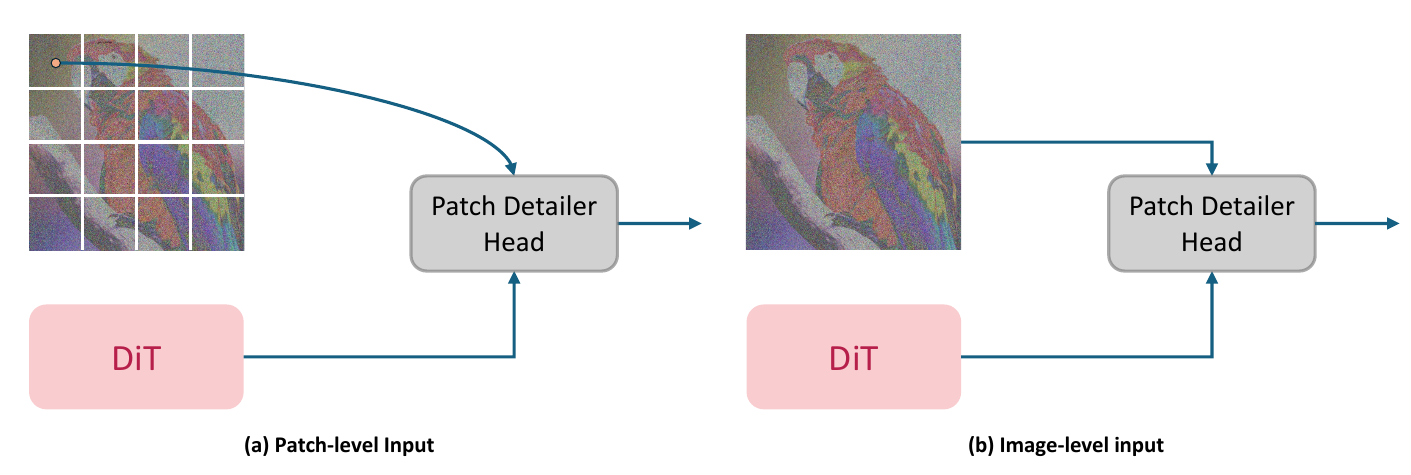}
    \caption{Different input formats of Patch Detailer Head.}
    \label{Fig: patch_and_image_input}
\end{figure*}

\begin{figure*}[h]
    \centering
    \includegraphics[width=\linewidth]{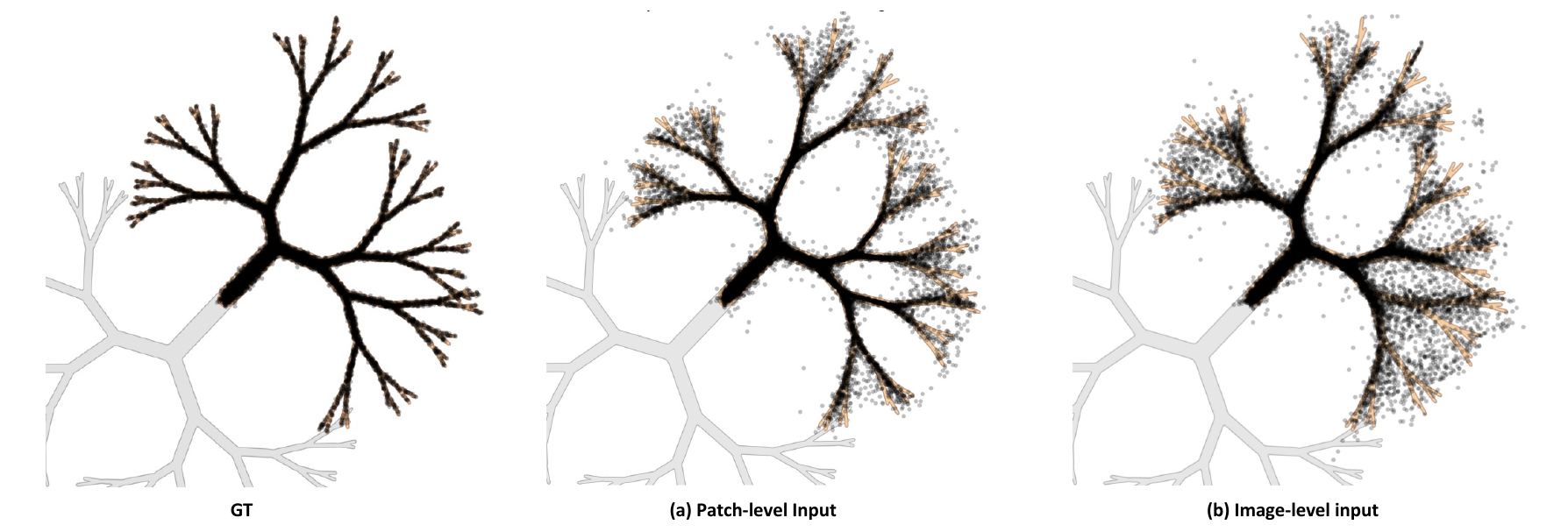}
    \caption{Toy experiment. (a) Visualization of manifold fitting with Patch-level input. The model precisely captures high-frequency branches. (b) Visualization of manifold fitting with Image-level input. The model exhibits over-smoothing and fails to resolve fine details.}
    \label{Fig:toy_exp}
\end{figure*}

\section{Alternative Patch Detailer Head (PDH).} 
Details of the alternative PDH are in Fig.~\ref{Fig:PDH_Variant}. The performance gap arises from the Convolutional U-Net's built-in inductive biases and hierarchical architecture, which better capture spatial detail and preserve local continuity. By contrast, alternative variants lack spatial information or are less effective at modeling local patterns, as explained in Sec. 3.4 (paper).
\begin{figure}[h]
    \centering
    \includegraphics[width=\linewidth]{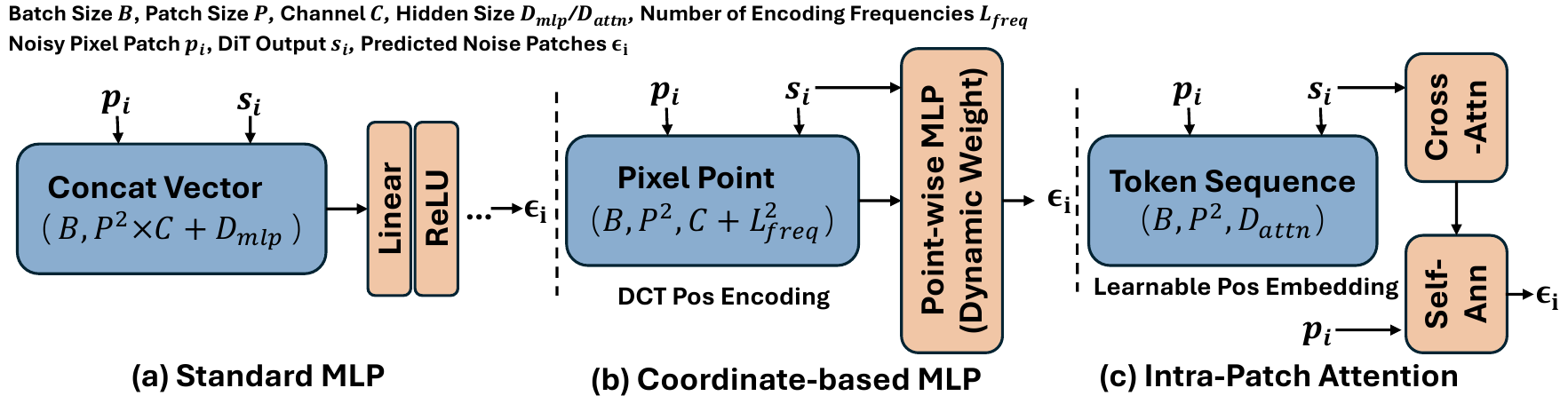}
    \caption{Details of other PDH.}
    \label{Fig:PDH_Variant}
\end{figure}



\section{More Visualization Results}
\begin{figure*}[h]
    \centering
    \includegraphics[width=\linewidth]{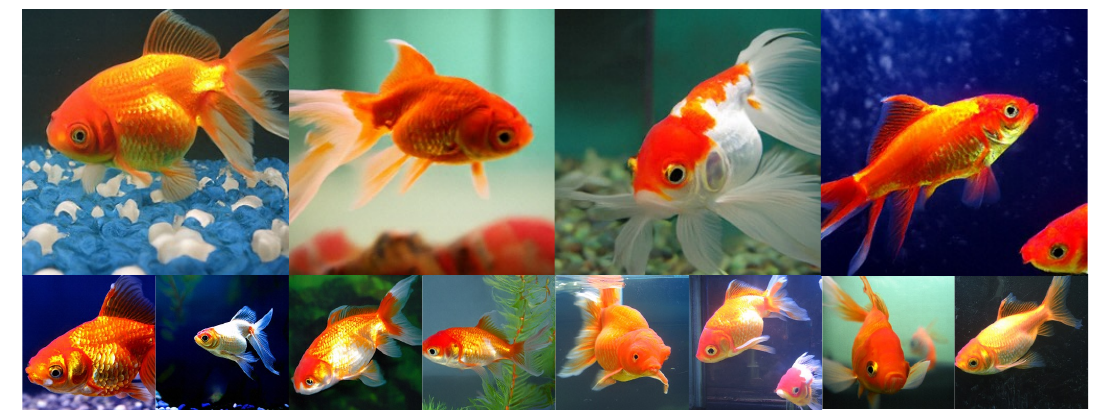}
    \caption{256$\times$256 samples. Class lable = “goldfish, Carassius auratus" (1). CFG = 4.0.}
    \label{Fig: 256_1}
    \vspace{-5mm}
\end{figure*}

\begin{figure*}[h]
    \centering
    \includegraphics[width=\linewidth]{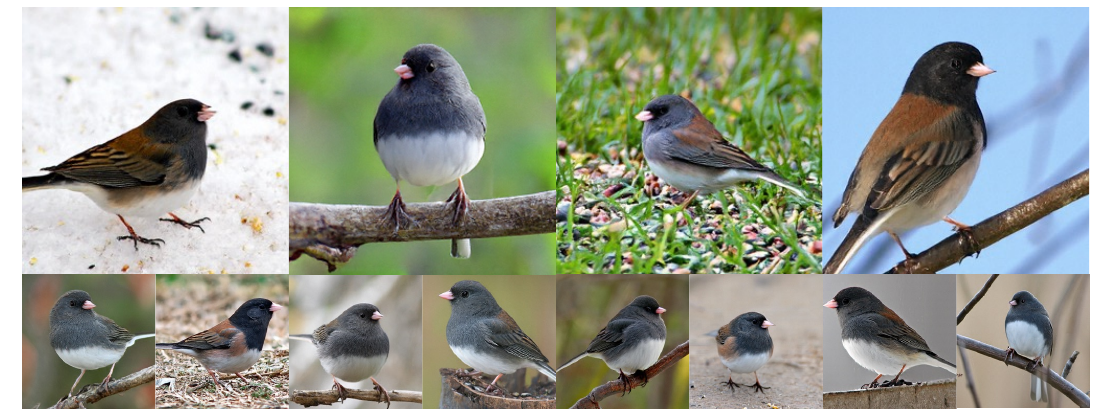}
    \caption{256$\times$256 samples. Class lable = “junco, snowbird" (13). CFG = 4.0.}
    \label{Fig: 256_2}
    \vspace{-5mm}
\end{figure*}

\begin{figure*}[h]
    \centering
    \includegraphics[width=\linewidth]{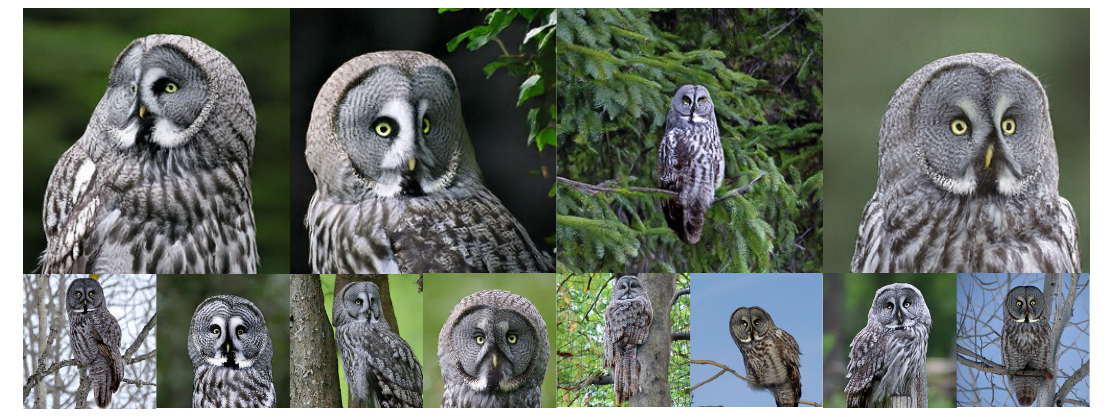}
    \caption{256$\times$256 samples. Class lable = “chickadee" (19). CFG = 4.0.}
    \label{Fig: 256_3}
    \vspace{-10mm}
\end{figure*}

\begin{figure*}[h]
    \centering
    \includegraphics[width=\linewidth]{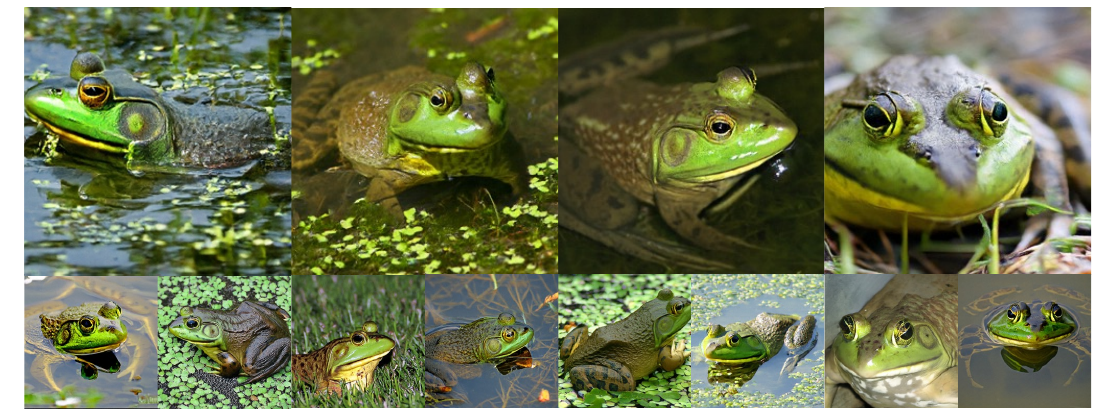}
    \caption{256$\times$256 samples. Class lable = “tree frog, tree-frog" (30). CFG = 4.0.}
    \label{Fig: 256_4}
\end{figure*}

\begin{figure*}[h]
    \centering
    \includegraphics[width=\linewidth]{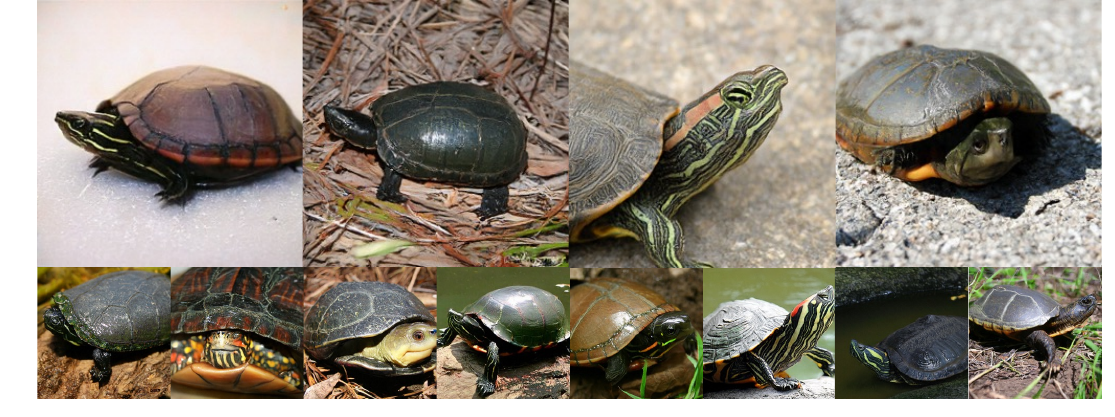}
    \caption{256$\times$256 samples. Class lable = “mud turtle" (35). CFG = 4.0.}
    \label{Fig: 256_5}
\end{figure*}

\begin{figure*}[h]
    \centering
    \includegraphics[width=\linewidth]{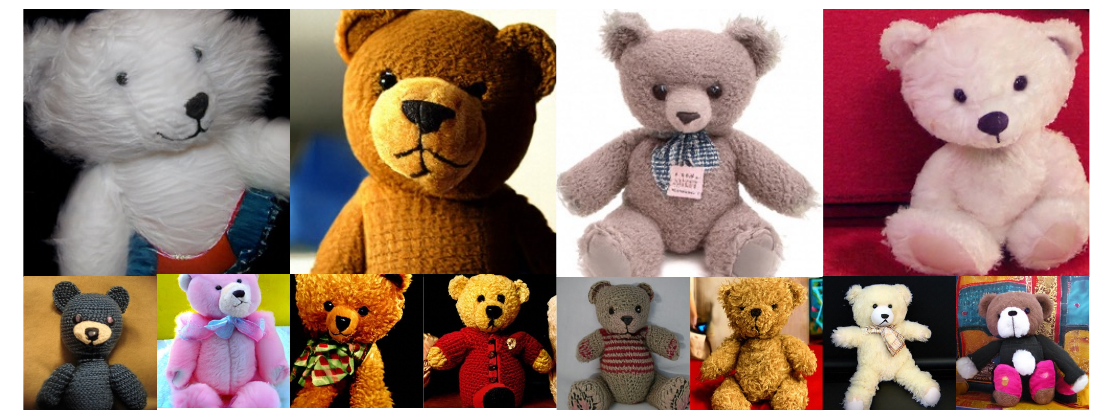}
    \caption{256$\times$256 samples. Class lable = “teddy, teddy bear" (859). CFG = 4.0.}
    \label{Fig: 256_9}
\end{figure*}

\begin{figure*}[h]
    \centering
    \includegraphics[width=\linewidth]{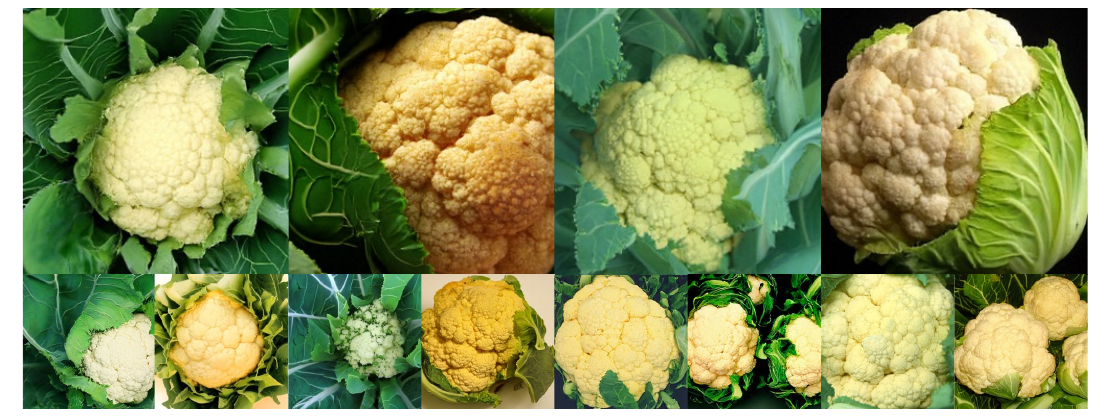}
    \caption{256$\times$256 samples. Class lable = “cauliflower" (938). CFG = 4.0.}
    \label{Fig: 256_8}
\end{figure*}

\begin{figure*}[h]
    \centering
    \includegraphics[width=\linewidth]{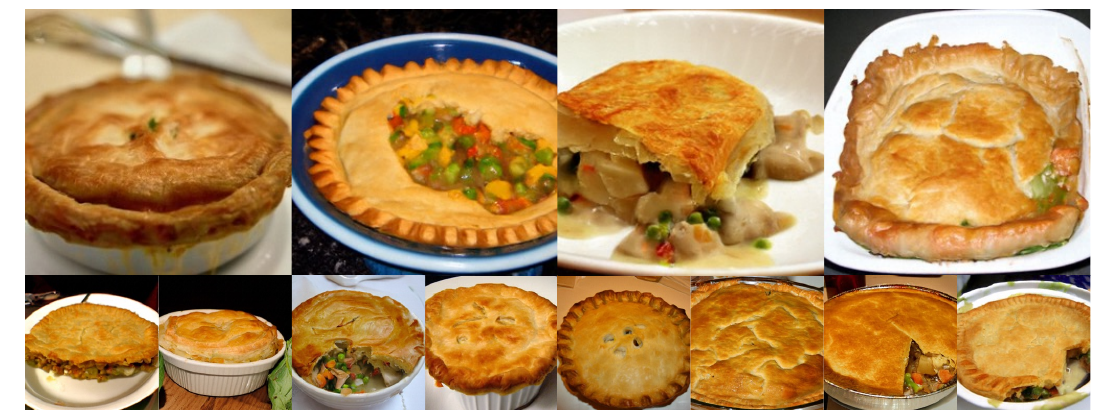}
    \caption{256$\times$256 samples. Class lable = “potpie" (964). CFG = 4.0.}
    \label{Fig: 256_7}
\end{figure*}

\begin{figure*}[h]
    \centering
    \includegraphics[width=\linewidth]{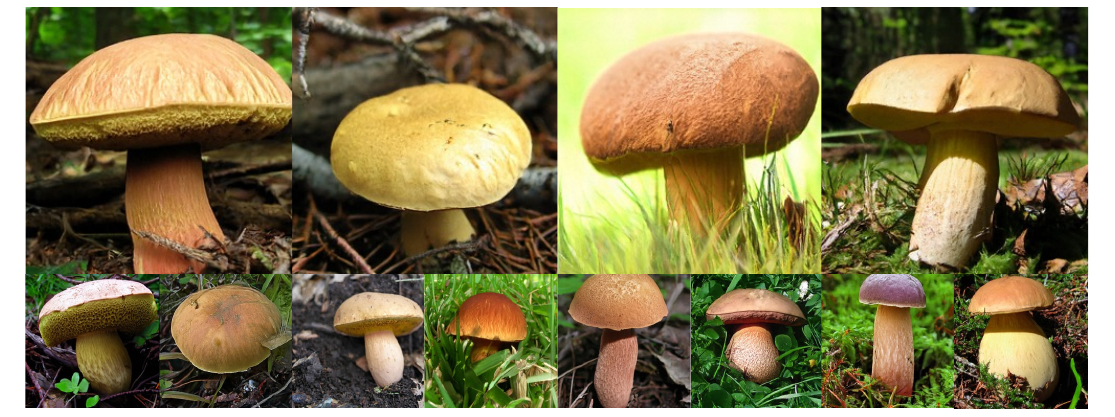}
    \caption{256$\times$256 samples. Class lable = “bolete" (997). CFG = 4.0.}
    \label{Fig: 256_6}
\end{figure*}

\begin{figure*}[h]
    \centering
    \includegraphics[width=\linewidth]{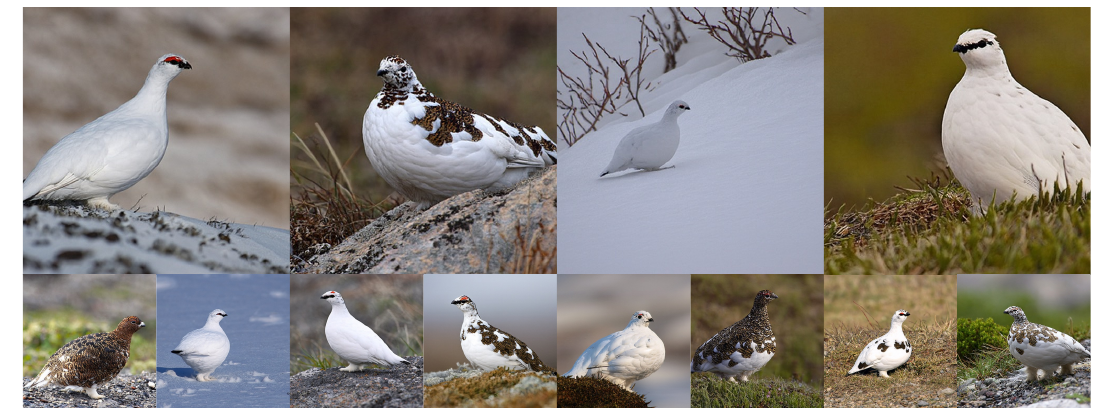}
    \caption{512$\times$512 samples. Class lable = “ptarmigan" (81). CFG = 4.0.}
    \label{Fig: 512_1}
\end{figure*}

\begin{figure*}[h]
    \centering
    \includegraphics[width=\linewidth]{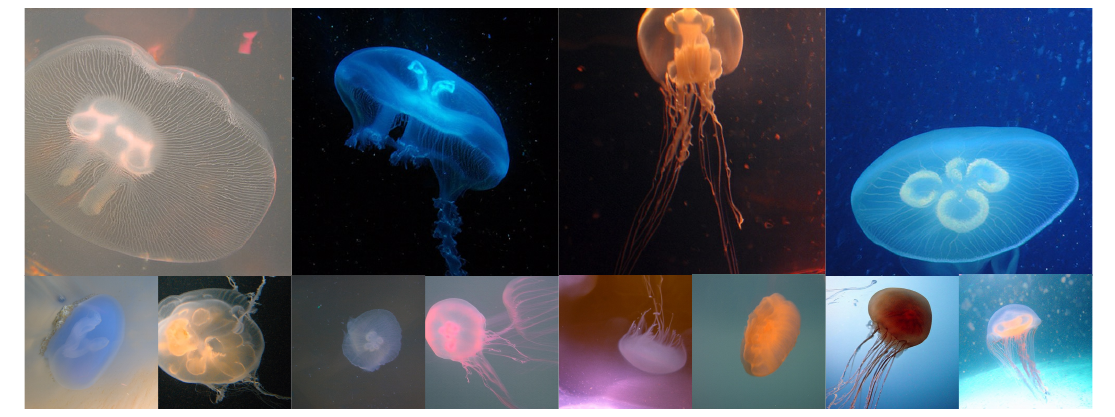}
    \caption{512$\times$512 samples. Class lable="jellyfish" (107). CFG=4.0.}
    \label{Fig: 512_2}
\end{figure*}


\begin{figure*}[h]
    \centering
    \includegraphics[width=\linewidth]{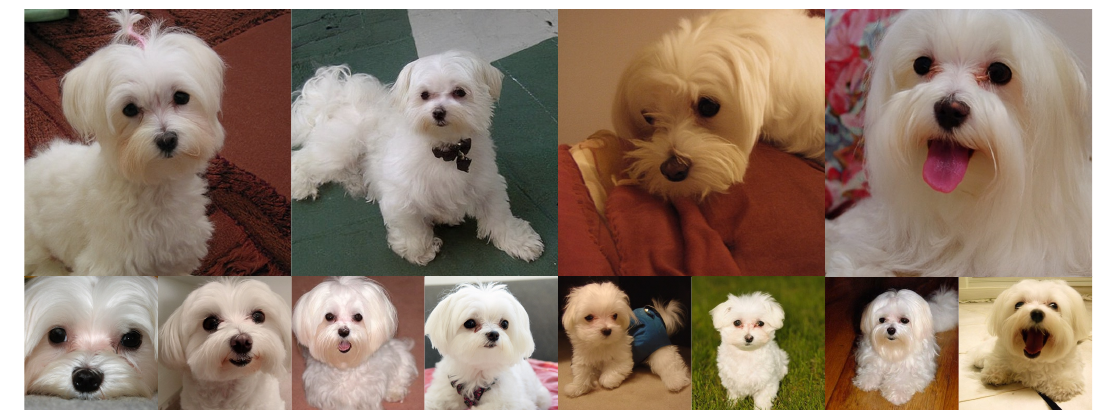}
    \caption{512$\times$512 samples. Class lable="Maltese dog, Maltese terrier, Maltese" (153). CFG=4.0.}
    \label{Fig: 512_4}
\end{figure*}


\begin{figure*}[h]
    \centering
    \includegraphics[width=\linewidth]{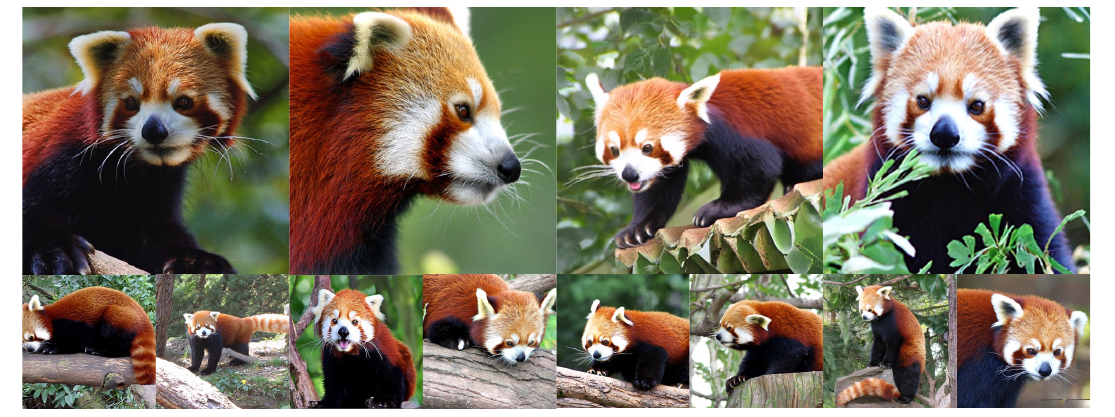}
    \caption{512$\times$512 samples. Class lable = “lesser panda, red panda, panda, bear cat, cat bear, Ailurus fulgens" (387). CFG = 4.0.}
    \label{Fig: 512_6}
\end{figure*}


\begin{figure*}[h]
    \centering
    \includegraphics[width=\linewidth]{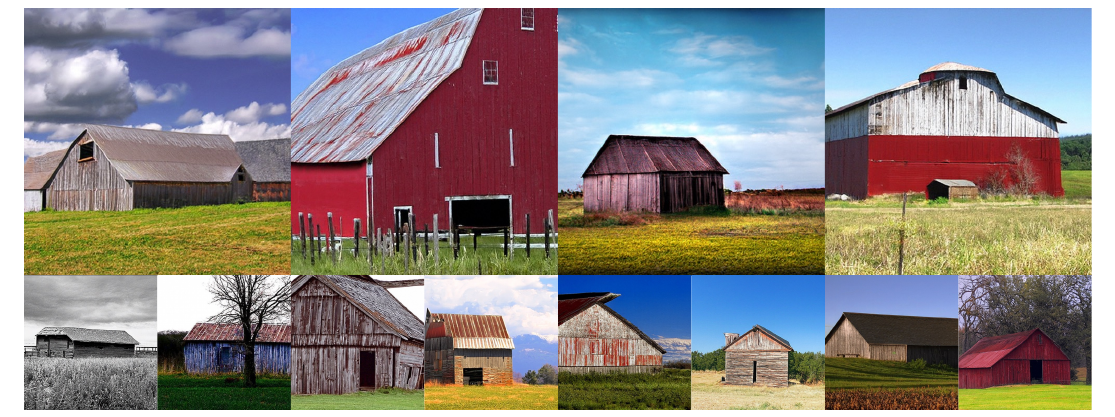}
    \caption{512$\times$512 samples. Class lable = “barn" (425). CFG = 4.0.}
    \label{Fig: 512_8}
\end{figure*}

\begin{figure*}[h]
    \centering
    \includegraphics[width=\linewidth]{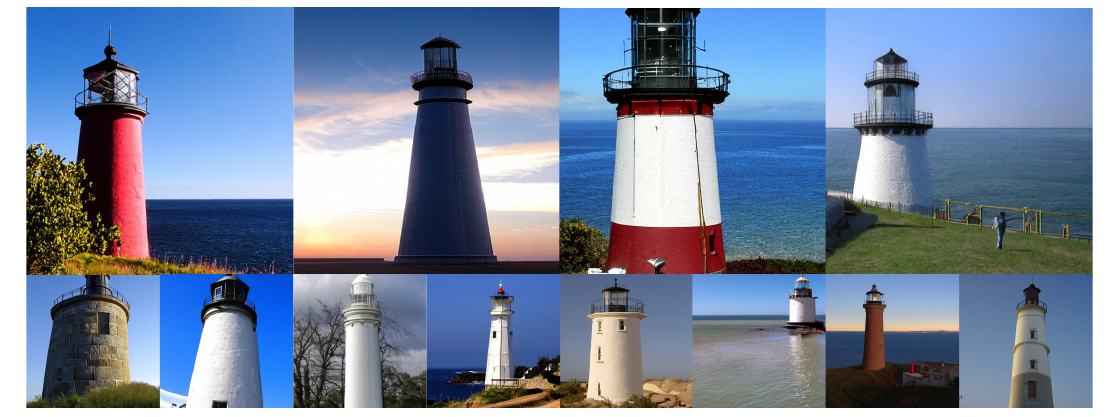}
    \caption{512$\times$512 samples. Class lable = “beacon, lighthouse, beacon light, pharos" (437). CFG = 4.0.}
    \label{Fig: 512_9}
\end{figure*}

\begin{figure*}[h]
    \centering
    \includegraphics[width=\linewidth]{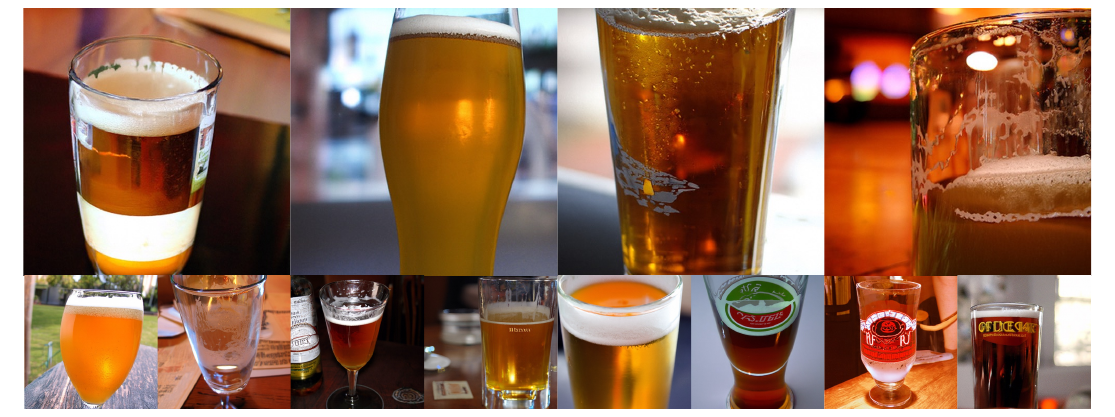}
    \caption{512$\times$512 samples. Class lable = “beer glass" (441). CFG = 4.0.}
    \label{Fig: 512_10}
\end{figure*}

\begin{figure*}[h]
    \centering
    \includegraphics[width=\linewidth]{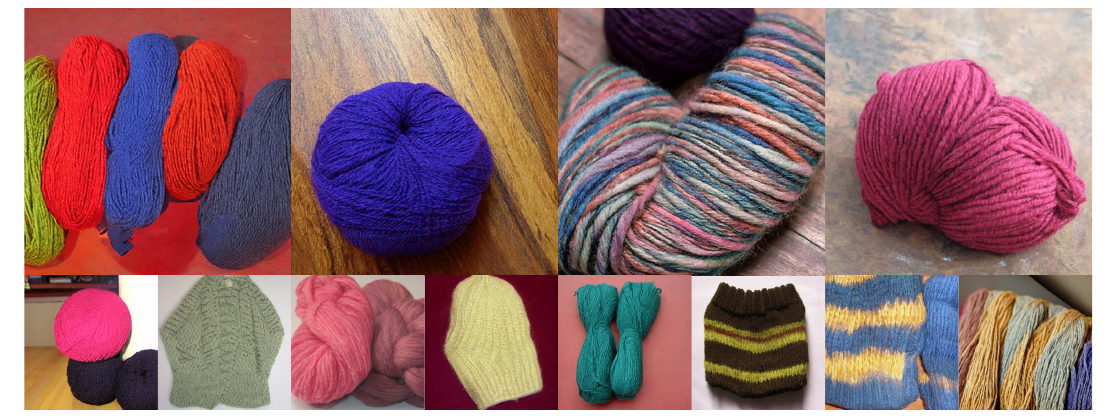}
    \caption{512$\times$512 samples. Class lable = “wool, woolen, woollen" (911). CFG = 4.0.}
    \label{Fig: 512_11}
\end{figure*}

\begin{figure*}[h]
    \centering
    \includegraphics[width=\linewidth]{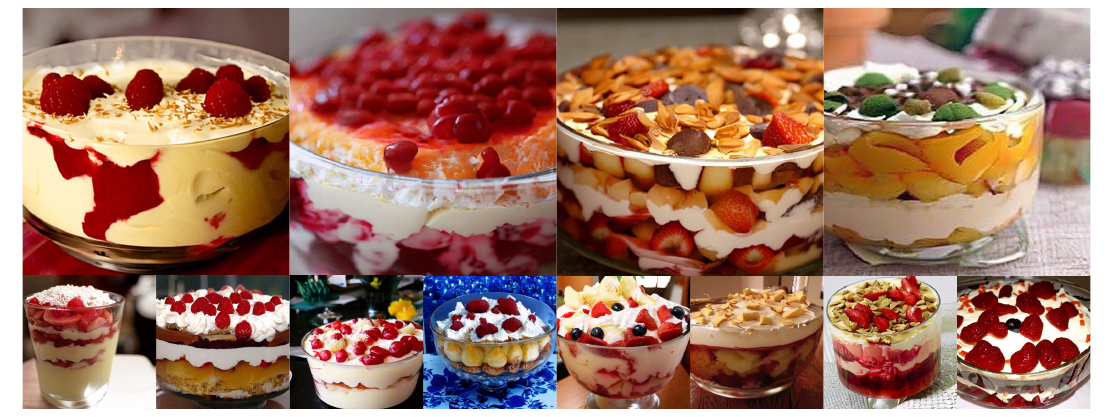}
    \caption{512$\times$512 samples. Class lable = “trifle" (927). CFG = 4.0.}
    \label{Fig: 512_12}
\end{figure*}

\end{document}